%% file: main.tex
\documentclass[10pt]{article}
\usepackage{amsmath,amsfonts}
\usepackage{algorithmic}
\usepackage{algorithm}
\usepackage{array}
\usepackage[caption=false,font=normalsize,labelfont=sf,textfont=sf]{subfig}
\usepackage{textcomp}
\usepackage{booktabs}
\usepackage{makecell}

\usepackage{stfloats}
\usepackage{url}
\usepackage{verbatim}
\usepackage{graphicx}
\usepackage{cite}
\usepackage{amsmath,amsfonts,amsthm,mathrsfs,url,amssymb,enumitem}
\usepackage{color}
\usepackage{fancyhdr}
\usepackage{amscd}
\usepackage[top=2.3cm, bottom=3.0cm, left=2.8cm,
right=2.8cm]{geometry}
\usepackage{bm}
\usepackage{enumitem,amsthm,amsmath,amsfonts,amssymb}
\usepackage{float}
\usepackage{mathtools}
\usepackage{algorithmic}
\usepackage{algorithm}
\usepackage{xcolor}
\usepackage{graphicx}
\usepackage{subcaption}
\usepackage{tikz}
\usepackage{multirow}
\usepackage{multicol}
\usepackage{hyperref}
\usepackage{mathrsfs, fancyhdr}
\usepackage{multirow}
\usepackage{booktabs}
\newtheorem{Def}{Definition}

\newtheorem{prop}{Proposition}

\newtheorem{thm}{Theorem}
\newtheorem{lemma}{Lemma}
\newtheorem{assump}{Assumption}

\newtheorem{remark}{Remark}
\numberwithin{equation}{section}
\newcommand{\R}{\mathbb{R}}
\newcommand{\N}{\mathbb{N}}

\newcommand{\F}{\mathcal{F}}
\newcommand{\E}{\mathcal{E}}

\newcommand{\X}{\mathcal{X}}
\newcommand{\Y}{\mathcal{Y}}
\newcommand{\Z}{\mathcal{Z}}
\newcommand{\D}{\mathcal{D}}
\newcommand{\Q}{\mathcal{Q}}
\newcommand{\W}{\mathcal{W}}
\newcommand{\calG}{\mathcal{G}}
\newcommand{\calH}{\mathcal{H}}
\newcommand{\M}{\mathcal{M}}

\newcommand{\black}{\textcolor{black}}

\newcommand{\avg}{\frac{1}{n} \sum_{i=1}^{n}}
\newcommand{\AVG}{\dfrac{1}{n} \sum_{i=1}^{n}}
\newcommand{\biavg}{\frac{1}{n(n-1)} \sum_{i\neq j}^n}
\newcommand{\ie}{i.e., }
\newcommand{\bE}{\mathbb{E}}
\newcommand{\cN}{\mathcal{N}}
\newcommand{\bx}{{\boldsymbol{x}}}
\hypersetup{hidelinks, colorlinks = true, linkcolor=blue, anchorcolor=red,
	citecolor=green}
\usepackage{tikz}
\usetikzlibrary{arrows.meta, positioning}
\usepackage[textsize=tiny]{todonotes}

\begin{document}

\title{Fine-grained Analysis of Non-parametric Estimation for Pairwise Learning
\thanks{This version corresponds to the accepted version in IEEE Transactions on Neural Networks and Learning Systems.}}

\author{Junyu Zhou$^1$\quad Shuo Huang$^2$ \quad Han Feng$^3$ \quad Puyu Wang$^{4\star}$\quad Ding-Xuan Zhou$^5$ \\ 
\smallskip \\
$^{1}$ Catholic University of Eichstätt-Ingolstadt, Ingolstadt, Germany\\
$^2$ Istituto Italiano di Tecnologia, Genova, Italy\\
$^3$ City University of Hong Kong, Hong Kong, China\\
$^{4}$ RPTU Kaiserslautern-Landau, Kaiserslautern, Germany\\
$^{5}$ The University of Sydney, Sydney,  Australia.
}

 \date{}

\maketitle

\begin{abstract}
 In this paper, we are concerned with the generalization performance of non-parametric estimation for pairwise learning. Most of the existing work requires the hypothesis space to be convex or a VC-class, and the loss to be convex. However, these restrictive assumptions limit the applicability of the results in studying many popular methods, especially kernel methods and neural networks.          
We significantly relax these restrictive assumptions and establish a sharp oracle inequality of the empirical minimizer with a general hypothesis space for the Lipschitz continuous pairwise losses. 
As an example, we apply our general results to study pairwise least squares regression and derive an excess population risk bound that matches the minimax lower bound for the pointwise least squares regression.  
The key novelty lies in constructing a structured deep ReLU neural network to approximate the true predictor, and in designing a targeted hypothesis space composed of networks with this structure and controllable complexity.  
Experiments validate the effectiveness of the proposed method.
This example demonstrates that the obtained general results indeed help us to explore the generalization performance on a variety of problems that cannot be handled by existing approaches. 
\end{abstract}

\parindent=0cm

\section{Introduction}\label{sec:Intro}
Pairwise learning is an important learning task that has attracted much attention in the modern machine learning society.
Many real-world tasks, such as ranking in search engines, recommendation systems, and metric learning, are inherently pairwise.
Unlike pointwise learning which studies a target model based on a pointwise loss function, 
pairwise learning refers to learning tasks in which the loss function takes a pair of examples as the input. 
Specifically, the goal of the pairwise learning is to learn a predictor $f:\X\times\X\to\R$ based on a pair of example $((x,y),(x',y'))$, whose performance is measured by a pairwise loss $\ell:\R\times\Y\times\Y\to\R^+$. Here, $\X\subset\R^d$ is the bounded closed input space and $\Y \subset \R$ is the output space.  Figure \ref{fig:tasks} illustrates the difference between pointwise learning and pairwise learning tasks.
Notable pairwise learning tasks include metric and similarity learning \cite{cao2016generalization,weinberger2009distance,ying2012distance}, AUC maximization \cite{kar2013generalization,ying2016stochastic,lei2021stochastic}, ranking \cite{agarwal,chapelle2010efficient,ranking}, and a minimum error entropy principle \cite{hu2012learning}. 
For instance, metric and similarity learning aims to study a metric $d(\cdot,\cdot)$ that estimates the distance or the similarity between a pair of observers.
The performance of $d$ on a pair example $((x,y),(x',y'))$ is measured by the loss $\ell(\tau(y, y')d(x, x'))$ with  $\tau(y, y') = 1$ if $y = y'$ and $ \tau(y, y')= -1$ else.
AUC maximization aims to rank positive instances above negative ones which involves a loss $\ell(w; (x,y), (x',y')) = ( 1-w^\top (x-x') )_+\,\mathbf{I}_{[y=1\wedge y'=-1]}$ with $x, x'\in  \mathbb{R}^d$ and $y,y'\in   \{\pm 1\}.$
It has been shown that pairwise learning has great advantages over traditional pointwise learning in modeling the relative relationships between sample pairs in many problems \cite{huai2019deep,kar2013generalization,wang2021differentially,zhao2011online}. 

\begin{figure}
    \centering
    \includegraphics[width=0.5\linewidth]{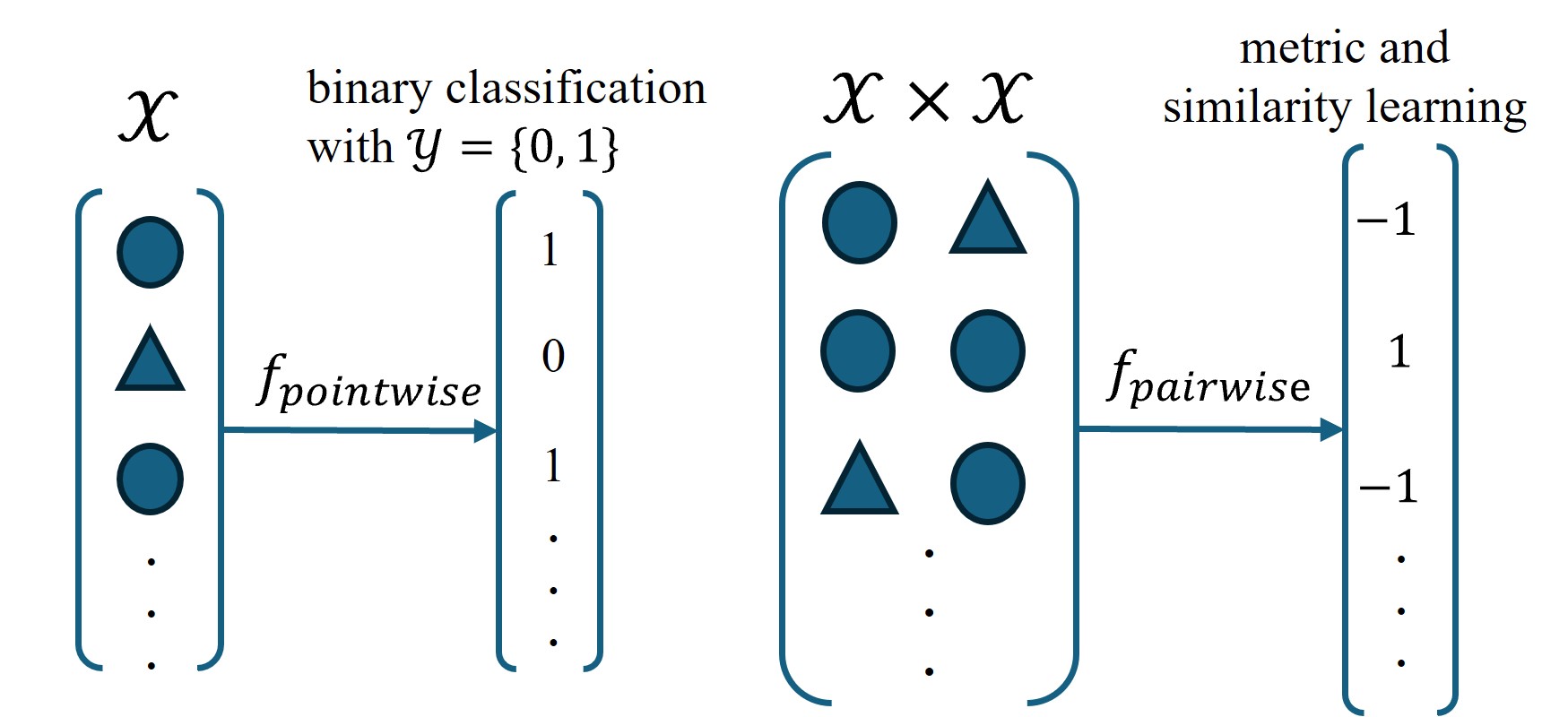}
    \caption{Pointwise learning vs. Pairwise learning}
    \label{fig:tasks}
\end{figure}

 {Parametric models usually refer to models whose architecture is fixed so that the number of parameters
does not grow with the sample size~\(n\). 
In contrast, non-parametric models, especially those based on structured deep networks \cite{schmidt2020nonparametric, yang2024nonparametric, bauer2019deep}, allows the architecture and the number of parameters grow
with~\(n\) or with the desired accuracy~\(\epsilon\).}
This is particularly valuable in pairwise settings, where the relation between input pairs and output preferences or scores can be highly complex and nonlinear. Besides, non-parametric learning allows us to explore minimax-optimal learning rates with mild assumptions on the true data distribution or model structure. Despite the widespread use of pairwise learning, theoretical foundations, particularly the study of non-parametric models, remain underdeveloped. In this paper, we focus on the generalization performance of non-parametric estimation for pairwise learning.

There has been a significant amount of work on the generalization analysis of nonparametric estimation for pairwise learning \cite{ranking, rejchel, zhou2024generalization}. However, most of these works rely on restrictive assumptions regarding both the loss function and the hypothesis space, and typically focus on specific learning problems. 
For instance, \cite{ranking} studied ranking problems and provided an oracle inequality for the empirical risk minimizer under the $0$–$1$ loss, achieving a convergence rate of $O\big((\frac{V\log(n)}{n})^{\frac{1}{2 - \beta}}\big)$, where $n$ is the training sample size, $\beta$ is the parameter of the variance condition and $V$ is the VC-dimension of the class of the ranking rules.
Here, the VC-dimension is the largest number of data points that can be labeled in all possible ways such that there exists a function in the class that correctly classifies each labeling, which measures the capacity or complexity of a class of functions.
The requirement that the class of ranking rules has a finite VC-dimension is quite restrictive, as many hypothesis spaces—especially reproducing kernel Hilbert spaces (RKHSs)—do not satisfy this condition. Consequently, \cite{ranking} are not applicable to widely used kernel-based methods.
Furthermore, the $0$–$1$ loss is rarely used in practice due to the associated minimization problem is known to be NP-hard even for a relatively simple class of functions such as linear function \cite{feldman2012agnostic}.
\cite{rejchel} derived upper bounds for the estimation error in ranking problems under convex, nonnegative loss functions, assuming that the hypothesis space is convex. However, this convexity assumption limits the applicability of their results to neural networks, whose hypothesis spaces are typically non-convex.
\cite{cao2016generalization,jin2009} established estimation error bounds of order $O(n^{-\frac{1}{2}})$ for metric and similarity learning problems. More recently, \cite{zhou2024generalization} obtained excess risk bounds for metric and similarity learning with the hinge loss, exploiting the specific structure of the underlying true metric.
Despite these advances, to the best of our knowledge, no existing work provides generalization analysis for pairwise learning with general loss functions under relaxed assumptions that accommodate modern hypothesis spaces, such as those used in neural networks or kernel methods.

On the other hand, numerous studies have applied neural networks to pairwise learning and ranking tasks \cite{koppel2025pairwise,he2022gnnrank,liu2025Trajectory,Duan2025toward}.
Nevertheless, a fundamental limitation of most existing models lies in their inability to capture non-transitive pairwise interactions, a phenomenon  observed in real-world scenarios but difficult to represent within the conventional score-difference framework. How to effectively model and learn from such non-transitive relationships remains an open and practically important question.

In this paper, we investigate comprehensive generalization analysis of pairwise learning under general settings by significantly removing the restrictive assumptions on both the hypothesis space and the loss functions.  
Specifically, we establish an oracle inequality of the empirical minimizer for a \textit{general hypothesis space} when the loss is Lipschitz continuous. Our results extend the existing literature in the following two aspects.  
First, the existing works either require the hypothesis space to be convex or a VC-class. As mentioned before, these assumptions are quit  restrictive and cannot applied to many popular methods including kernel methods and neural networks. 
Second, we consider the general losses that are Lipschitz continuous and possible nonconvex.
Various common-used surrogate losses including the hinge loss, least squares loss and logistic loss satisfy these assumptions, which makes our results applicable to a wide range of pairwise learning problems including ranking, pairwise regression and metric and similarity learning.  Furthermore, our experiments demonstrate that the proposed model is particularly well suited for handling non-transitive pairwise relationships, providing a positive answer to the above mentioned question.

As an example, we apply our results to the study of pairwise least squares regression using deep ReLU networks. By leveraging our main theorem, we establish minimax-optimal excess population risk bounds of order $O\big(n^{-\frac{2r}{2r + d}}\big)$ for pairwise least squares regression. Here, $d$ is the dimension of the input space and $r\in\N$ is the smoothness index of the target function of the pointwise least squares regression. 
Our key idea is to construct a novel structured deep ReLU network that approximates the true predictor $f_\rho$ by exploiting the specific anti-symmetric structure of $f_\rho$, and design the targeted hypothesis space that consisting of the networks with this structure. 
From a practical perspective, we provide a way to build such a structured deep ReLU network, which achieves superior generalization performance compared to standard benchmarks for the pairwise regression problem. Experimental results validate the effectiveness and advantages of our approach.
This application shows that the general oracle inequality greatly helps us explore the generalization performance on a variety of problems that cannot be handled by existing approaches.

\begin{figure*}[t]
\centering
\begin{tikzpicture}[
    scale=0.8, transform shape,
    node distance=0.8cm and 1.5cm,
    every node/.style={font=\sffamily},
    lemma/.style={draw, rectangle, rounded corners, fill=white!95!blue!10, inner sep=4pt},
    ->, >=Stealth
  ]
  \node[lemma] (A1) {Assumption~\ref{assump:uniformly_bounded}};
  \node[lemma, right=of A1] (A2) {Assumption~\ref{assump:Lipschitz}};
  \node[lemma, right=of A2] (A3) {Assumption~\ref{assump:cover1}};
  \node[lemma, right=of A3] (A4) {Assumption~\ref{assump:cover2}};
  \node[lemma, below=of A2] (T1) {Theorem~\ref{thm:oracle} (\textbf{Main result})};
  \node[lemma, right=of T1] (T2) {Theorem~\ref{thm:oracle_structured_NNs}};
  \node[lemma, right=of T2] (T4) {Theorem~\ref{thm:excess}};
  \node[lemma, below=of T4] (T3) {Theorem~\ref{thm:approx_error}};
  \node[lemma, below=of T1] (A5) {Assumption~\ref{assump:anti-sym}};
  \node[lemma, below=of T2] (A6) {Assumption~\ref{assump:sobolev}};

  \draw (A1) -- (T1);
  \draw (A2) -- (T1);
  \draw (A3) -- (T1);
  \draw (A4) -- (T1);
  \draw (A5) -- (T2);
  \draw (T1) -- (T2);
  \draw (A6) -- (T3);
  \draw (T2) -- (T4);
  \draw (T3) -- (T4);
\end{tikzpicture}
\caption{Diagram of the relationships among theorems and assumptions}
\label{fig:proof-structure}
\end{figure*}

\medskip

\noindent\textbf{Organization of the paper.} The rest of the paper is organized as follows. Section~\ref{sec:pre} introduces some basic concepts and definitions. 
Section~\ref{sec:main} presents the main  result: an oracle inequality for the empirical minimizer over a general hypothesis space, under the assumption that the pairwise loss is Lipschitz continuous and considering various capacity conditions of the hypothesis space.
In Section~\ref{sec:relu}, we apply our main results to analyze the generalization performance of pairwise learning by constructing a structured deep ReLU network model.
In particular, we derive minimax-optimal excess risk rates by leveraging the main theorem.
In Section~\ref{sec:exper}, we present experimental results to demonstrate the effectiveness of our structured deep ReLU model proposed in Section~\ref{sec:relu}. 
We conclude the paper in Section~\ref{sec:conclusion}. 
To provide a clearer understanding of the paper’s theoretical structure, Figure~\ref{fig:proof-structure} illustrates the relationships among our theorems and assumptions.  
We give Table~\ref{tab:summary} which summarizes the main notations. 
\begin{table}[htbp]
\small
\setlength{\abovecaptionskip}{4pt plus 2pt minus 2pt}
 \setlength{\belowcaptionskip}{4pt plus 2pt minus 2pt}
\setlength{\tabcolsep}{4pt}
 \centering\def\arraystretch{1.3}
\centering
  \begin{tabular}{|c|c|}\hline
  Symbol & Meaning \\ 
  \hline
  $\X/\Y$ & input/output space\\
  $\calH$& hypothesis space\\
   $d$ & dimension of the input space \\
   $n$ & training sample size \\
   $r$ & smoothness index of the true predictor\\
  $\E(f)/\E_z(f)$ & population/empirical risk\\
  $ \hat{f}_z$ & $\arg\min_{f\in\calH} \E_z(f)$ \\
  $f_\calH$ &  $\arg\min_{f\in\calH} \E(f)$ \\
  $f_\rho$&  true predictor \\
  $\E(\hat{f}_z) - \E(f_\rho)$  &  excess population risk \\
   $\|f \|_{L^\infty(\X\times\X)}$ & $  \sup_{x,x'\in\X} |f (x,x')| $ \\
   $\|f\|_{L^2_\rho} $& $ (\int |f|^2 d\rho_{\bx})^{1/2}$\\
  \hline
  \end{tabular}%
  \normalsize
  \caption{Summary of Main Notations.
		\label{tab:summary}} 		
\end{table}

\section{Learning Setting and Preliminaries}\label{sec:pre}
Let $\rho$ be a population distribution defined on the sample space  $\Z=\X\times\Y$, where $\X\subset\R^d$ is the bounded closed input space and $\Y \subset \R$ is the output space. 
Let $\ell:\R\times\Y\times\Y\to\R^+$ be a pairwise loss function. Given a predictor $f:\X\times\X\to\R$, the population risk $\E(f)$ with the loss $\ell$ is defined as
 \begin{align*}
     \E(f)  :&= \int_{\Z\times\Z} \ell(f(x,x'),y,y') d\rho(z)d\rho(z')\\
        &= \bE[\ell(f(X,X'),Y,Y')].
 \end{align*}
Given a sample $S=\{Z_i=(X_i,Y_i)\}_{i=1}^n$ independently drawn from $\rho$, define the empirical  risk based on  $S$ by     $\E_z(f) := \biavg \ell(f(X_i,X_j),Y_i,Y_j).$ 
Let $f_\rho = \arg\min_{f\in\F} \E(f)$ be the true predictor (the target function) that minimizes the population risk over the space $\F$ consisting of all measurable functions from $\X\times\X$ to $\R$, and $\hat{f}_z = \arg\min_{f\in\calH} \E_z(f)$ be the empirical minimizer with the hypothesis space $\calH$.
In this paper, we are interested in studying the statistical generalization performance of $\hat{f}_z$, which is measured by the \textit{excess population risk} $\E(\hat{f}_z) - \E(f_\rho)$, \ie the distance between the expected error of $\hat{f}_z$ and the least possible error $\E(f_\rho)$. We consider using the following error decomposition to study the excess population risk
\begin{align}\label{eq:error_decomp}
        \E(\hat{f}_z) \!-\! \E(f_\rho) &= \{\E(\hat{f}_z) \!-\! \E(f_\calH)\} \!+\! \{\E(f_\calH) \!-\! \E(f_\rho)\},\nonumber\\
        &=: S(\calH) + D(\calH),
    \end{align}
    where $f_\calH = \arg\min_{f\in\calH} \E(f)$. The terms $S(\calH)$ and $D(\calH)$ are called the estimation error and the approximation error, respectively, and will be estimated part by part. Figure \ref{fig:error} illustrates this decomposition of the error. 
 Throughout this paper, we denote by $C_{\alpha,\beta,\gamma}$ the constant depends on parameters $\alpha,\beta$ and $\gamma$, and denote $C$ an absolute constant. These constants may differ from line to line and are always assumed to be greater than or equal to $1$.  We denote $\mathbf{1}_{[A]}$ the indicator function that takes value $1$ if the event $A$ happens and $0$ otherwise. Let $\lceil C \rceil$ denote the least integer number greater than or equal to $C$. For a measurable set $A\subset\X$, let $\rho_\bx(A):= Prob\{A\} = \bE_X[\mathbf{1}_{[A]}]$ and $\rho(A|x):= Prob\{A|X=x\} = \bE[\mathbf{1}_{[A]}|X=x]$.

\begin{figure}
    \centering
\includegraphics[width=0.5\linewidth]{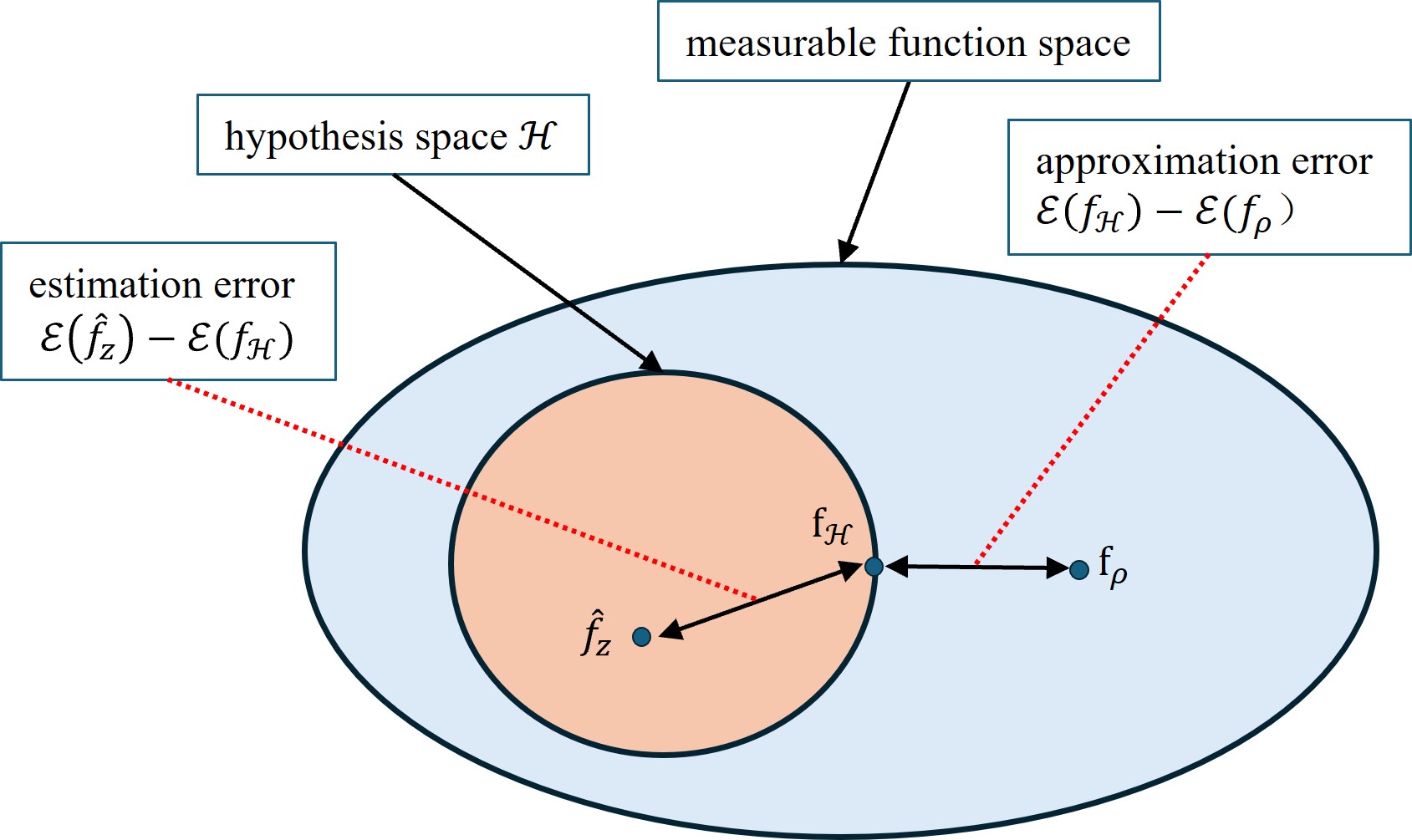}
    \caption{Error decomposition}
    \label{fig:error}
\end{figure}
        
Our analysis requires some standard assumptions on the true predictor $f_\rho$, the loss function $\ell$ and the capacity of the hypothesis space.     \begin{assump}\label{assump:uniformly_bounded}
		There exists a constant $\eta > 0$ such that	$\|f_\rho\|_{L^\infty(\X\times\X)}:= \sup_{x,x'\in\X} |f_\rho(x,x')| \le \eta.$ 
	\end{assump}
Most important problems satisfy this assumption. We use pairwise least squares regression and ranking problems as two examples here.   

\noindent\textit{Example 1. Pairwise least squares regression.}
    For the pairwise least squares regression with loss $\ell(f(x,x'),y,y') = (f(x,x') -y + y')^2$, it is known that the true predictor has the form $f_\rho(x, x')= \tilde{f}_\rho(x) - \tilde{f}_\rho(x')$, where $\tilde{f}_\rho(x):= \bE[Y|X = x]$ is the target function (Bayes rule) for the pointwise least squares problem \cite{online_pairwise}. One often assumes that the distribution of $Y$ is bounded by a constant $B>0$, \ie $Prob\{|Y| \le B\} = 1$, then we can show that Assumption~\ref{assump:uniformly_bounded} holds with $\eta = 2B$ by noting $\|f_\rho\|_{L^\infty(\X\times\X)} \le 2\|\tilde{f}_\rho\|_{L^\infty(\X)} \le 2B$. 
    
\noindent\textit{Example 2. Ranking problems.}
For ranking problems with the hinge loss $\ell(f(x,x'),y,y') = (1 - sgn(y - y')f(x,x'))_+$, the work \cite{HZFZ} proved that the true predictor has the form $f_\rho(x,x') = sgn(T(x,x') - T(x',x))$, where $T(x,x') = Prob\{Y > Y'|X = x, X' = x'\}$. Then it is easy to show that $\eta =  1$ in this case since $|sgn(t)| \le 1$ for any $t\in\R$.

We also need to make the following assumptions about the Lipschitz continuity of the loss function. Here, we only consider the Lipschitz property of $\ell$ over $\big[-\|f_\rho\|_{L^\infty(\X\times\X)}, \|f_\rho\|_{L^\infty(\X\times\X)}\big]$ since the values predicted by $f_\rho$ always lie on this interval.
\begin{assump}\label{assump:Lipschitz}
		There exists a constant $K>0$, such that for any    $t_1, t_2 \in [-\|f_\rho\|_{L^\infty(\X\times\X)}, \|f_\rho\|_{L^\infty(\X\times\X)}]$   and  $y, y' \in \Y$, there holds
      \[\left|\ell(t_1, y, y') - \ell(t_2, y, y')\right| \le K\left|t_1 - t _2\right|.\]
	\end{assump}
Assumption~\ref{assump:Lipschitz} holds for the hinge loss with $K = 1$ and the pairwise least squares loss with $K = 8B$ if the distribution of $Y$ is bounded by $B>0$.

With a little abuse of notation, for any probability measures $\rho$, we denote $L^2_\rho$ as the metric induced by the norm $\|\cdot\|_{L^2_\rho}$. Here, $\|\cdot\|_{L^2_\rho}$ is $L^2$ norm defined by $\|f\|_{L^2_\rho} = (\int |f|^2 d\rho_{\bx})^{\frac{1}{2}}$. 
Recall that $\rho_{\bx}$ is the marginal distribution of $\rho$ on $\X$, we define two empirical probability measures based on the observed sample $S$ as $\mu_n := \avg \delta_{X_i}$ and $\nu_n:= \biavg \delta_{(X_i,X_j)}$, where $\delta_{(\cdot)}$ is the counting measure. For any $f\in\calH$, we define the norms
    \begin{align} \big\|f\big\|_{L^2_{\rho_{\bx}\times\mu_n}} &:= \Big(\avg\bE\big[f(X,X_i)\big|X_i\big]^2\Big)^{\frac{1}{2}}, \label{def:mu_n}\\  \big\|f\big\|_{L^2_{\nu_n}} &:= \Big(\biavg f(X_i,X_j)^2\Big)^{\frac{1}{2}}. \label{def:nu_n}
    \end{align}
Our analysis needs the assumptions on the capacity of the hypothesis space $\calH$, which are measured by different conditions of the associated covering numbers.
The definition of covering number is given as follows. 
	\begin{Def}[Covering number \cite{HDP}]\label{def:covering}
		Let $(\mathcal{T},d)$ be a metric space. Consider a subset $\mathcal{K} \subset \mathcal{T}$ and let $\epsilon > 0$. A subset $\mathcal{N} \subset \mathcal{K}$ is called an $\epsilon$-net of $\mathcal{K}$ if every point in $\mathcal{K}$ is within a distance $\epsilon$ of some point of $\mathcal{N}$, i.e.
		$\forall\ x \in \mathcal{K}, \  \exists\ x_0 \in \mathcal{N} : d(x, x_0) \le \epsilon.$
		The smallest possible cardinality of an $\epsilon$-net of $K$ is called the covering number of $\mathcal{K}$ and is denoted by $\mathcal{N}(\mathcal{K},d,\epsilon)$.
	\end{Def}
\begin{assump}\label{assump:cover1}
        There exist constants $s_1,s_2\ge e$ and $V_1,V_2 > 0$ such that for any $\epsilon\in(0,1]$, there holds
        \begin{align*}    \cN\!\left(\calH,L^2_{\!\rho_{\bx}\times\mu_n}\!\!, \epsilon\right) \!\le\! s_1\Big(\frac{1}{\epsilon}\Big)^{\!V_1}\!\!, \, \cN\!\left(\calH,L^2_{\!\nu_n}\!,\epsilon\right) \!\le\! s_2\Big(\frac{1}{\epsilon}\Big)^{\!V_2}\!\!.
        \end{align*}
    \end{assump}
   Any bounded subset of a finite-dimensional space \cite{HDS} and the space consisting of bounded deep neural networks \cite{schmidt2020nonparametric,VC} satisfy Assumption~\ref{assump:cover1}.  However, many important spaces do not satisfy this assumption, e.g., RKHSs. We introduce the following assumption as an alternative to Assumption~\ref{assump:cover1} to handle this setting. 
    \begin{assump}\label{assump:cover2}
        There exist constants $s_1',s_2' \ge 1$ and $V_1',V_2' \in (0,1)$ such that for any $\epsilon\in(0,1)$,
        \begin{align*}       \log\left(\cN\left(\calH,L^2_{\rho_{\bx}\times\mu_n},\epsilon\right)\right) \le s_1'\Big(\frac{1}{\epsilon}\Big)^{V_1'},\\          \log\left(\cN\left(\calH, L^2_{\nu_n},\epsilon\right)\right) \le s_2'\Big(\frac{1}{\epsilon}\Big)^{V_2'}.
        \end{align*}
    \end{assump}
    We further introduce the Pseudo-dimension \cite{neuralnetworklearning, VC} of the capacity of the hypothesis space, which is defined by VC-dimension of a subgraph set.
    For a formal definition of VC-dimension, one can see Definition 9.6 in \cite{gyorfi}.
   
        \begin{Def}[Pseudo-dimension \cite{neuralnetworklearning, VC}]\label{def:pseudo}
		Let $\F$ be a class of functions $f : \X\times\X \to \R$ and $\F^+:= \{(x,x',t): f(x,x') > t, f \in \F\}$ be its subgraph set, then the pseudo-dimension $Pdim(\F)$ of $\F$ is defined as
		 $Pdim(\F):= VC(\F^+)$,
		where $VC(\F^+)$ is the VC-dimension of $\F^+$. Further, if $Pdim(\F) < \infty$, then we call $\F$ a VC-class.
	\end{Def}

	\black{The following lemma reveals a relation between the covering number and pseudo-dimension \cite[Theorem 2.6.7]{empirical} and will be used frequently later. 
	\begin{lemma}[\!\!{\cite[Theorem 2.6.7]{empirical}}]\label{lemma:VC_class}
		For a VC-class $\F$ of functions with uniform bound $F$, one has for any probability measure $\rho$, 
		$$\mathcal{N}(\!\F\!, \!L^2_\rho,\!\epsilon F)\! \le\! C\, Pdim(\!\F\!)(16e)^{\!Pdim(\F)} \!\big(\!\frac{1}{\epsilon}\!\big)^{\!2(Pdim(\!\F\!) \!-\! 1)}$$
		for an absolute constant $C\!>\!0$ and $0 \!<\! \epsilon \!<\!1$, where $L^2_\rho$ denotes the $L^{2}$ norm with respect to $\rho_{\bx}$.
	\end{lemma}
    Lemma \ref{lemma:VC_class} implies that any bounded VC-class satisfies Assumption \ref{assump:cover1}. Specifically, for any hypothesis space $\calH$ which is a VC-class and uniformly bounded by $\eta>0$,  $\calH$ satisfies Assumption \ref{assump:cover1} with $s_1=s_2 = C_{\eta,Pdim(\calH)}$ and $V_1=V_2=2(Pdim(\calH)-1)$, respectively. Conversely, if we further assume that the inequalities therein hold with any probability measures, then Assumption \ref{assump:cover1} can be employed as an alternative definition of a VC-class.}

\section{Main Results}\label{sec:main}

In this section, we present our main result on an oracle inequality of $\hat{f}_z$.
 {To obtain a fast convergence rate, we require the following variance condition (the second order property) of a function class \cite{EM_Bartlett}.
	\begin{Def}[Variance-expectation bound \cite{EM_Bartlett}]\label{def:variance-expectation}
		Let $M>0$ and $\beta \in[0,1]$. Let $\F$ be a function class consisting of functions $f:\Z\times\Z\to\R$ with nonnegative first moment, \ie $\bE[f] \ge 0$ for any $f\in\F$. We say that $\F$ has a variance-expectation bound with parameter pair $(\beta, M)$, if for any $f \in \F$,
        \begin{align}\label{eq:variance}
            \bE[f^2] \le M(\bE[f])^\beta.
        \end{align}
	\end{Def}
        In statistical learning theory, we often set $\F$ to be the shifted hypothesis space $\{\ell(f) - \ell(f^*): f \in \calH\}$, where $f^*$ is either the true predictor $f_\rho$ or the oracle $f_\calH$ (best predictor in $\calH$).
        The inequality \eqref{eq:variance} holds with $\beta = 0$ for any uniformly bounded function class.
        There are numerous cases in which \eqref{eq:variance} also holds with $\beta\in(0,1]$.
        For instance, \eqref{eq:variance} is satisfied for ranking \cite{ranking} and for binary classification \cite{boucheron2005theory} if the distribution $\rho$ over $\Z$ satisfies a low-noise condition.
        \cite{convexity} showed that \eqref{eq:variance} holds with $\beta = \min\{1,2/r\}$ if the hypothesis space $\calH$ is convex and the modulus of convexity $\delta$ of the loss satisfies $\delta(\epsilon) \ge c\epsilon^r$.
        Further, we will prove later \eqref{eq:variance} is satisfied with $\beta = 1$ for pairwise least squares regression when the distribution $\rho$ is bounded (see Lemma 10 in Appendix B for more details).}

Note in the following theorem, we remove several restrictive assumptions commonly used in existing literature, such as the requirement that the hypothesis space be convex \cite{convexity,rejchel} or form a VC class \cite{ranking}.
Specifically, part (b) of the theorem is established for non VC-class, and the whole theorem holds for arbitrary nonconvex hypothesis spaces satisfying mild capacity conditions.
The proof is given in Appendix A.  We denote $A \bigvee B= \max\{A,B\}$. 
	\begin{thm}\label{thm:oracle}
		Suppose Assumptions $\ref{assump:uniformly_bounded}$ and $\ref{assump:Lipschitz}$ hold. Let $\calH$ be an arbitrary hypothesis space with uniform bound $\eta > 0$ such that for any $f\in\calH\cup\{f_\rho\}$ and almost $z,z'\in\Z$, there holds
        \begin{align}\label{eq:symmetry}
            \ell(f(x,x'),y,y') = \ell(f(x',x),y',y).
        \end{align}
        Let $M \!>\! 0$ and $\beta\!\in\![0,1]$, suppose that the shifted hypothesis space $\F\!:=\!\{\ell(f(x,x'), y, y') - \ell(f_\rho(x,x'), y, y'): f \in \calH\}$ has a variance-expectation bound \eqref{eq:variance} with parameter pair $(\beta, M)$. The following statements hold true.  
        \begin{enumerate}[label=(\alph*), leftmargin=*]        \setlength\itemsep{-2mm}  
            \item If the capacity of $\calH$ satisfies Assumption \ref{assump:cover1}, then for any $\delta \in (0, 1/2)$, with probability at least $1 - \delta$,  
            \begin{align*}
         \!\!\!\!\!\!\!\!\!  \E\!(\!\hat{f}_z \!) \!-\! \E\!(\!f_\rho\!)\!&\le  \!C_{\!\eta\!,K\!,M\!,\beta}  \!\Big(\!\frac{(V_1\!\!\bigvee\log (s_1\!)\!)\!\log (n)}{n}\!\Big)^{\!\frac{1}{2\!-\!\beta}}\!\!\log( \frac{4}{\delta} )\nonumber\\
         &\!\!\!\!\!+\!  \frac{\!C_{\eta,K}\! \log^2\! (\frac{\delta}{2}) \max\{V_1,\!V_2,\log(s_1),\!\log(s_2)\}}{n}\\
         &\!\!\!\!\!+  ( \beta  +  2 ) \big( \E(f_\calH ) \!-\! \E(f_\rho ) \big).
        \end{align*}
        \item If the capacity of $\calH$ satisfies Assumption \ref{assump:cover2}, then for any $\delta\in(0,1/2)$, with probability at least $1 -\delta$,  
        \begin{align*}
           \!\!\!\!\!\!  \E(&\hat{f}_z)  \! - \! \E(f_\rho) \!\le\!   \Big(C_{\eta\!,K\!,M\!,\beta\!} \max\Big\{\!\! \sqrt{\!s_1'}\big(\frac{1}{n}\big)^{\!\frac{2}{(2\!+\!V_1')(2\!- \beta)}},  \\ &\!\!\Big(\!\frac{\log (n) }{n}\!\Big)^{ \!\frac{1}{2-\beta}}\!\log(\frac{4}{\delta}) \!\Big\}\!+\!  \frac{\!C_{\eta,K}\max\{s_1',\!s_2'\}\log^2\!(\frac{\delta}{2})}{n\big(1 \!-\! \max\{V_1',V_2'\}\big)}\Big)\\
            & + \big(\beta + 2\big) \big(\E(f_\calH) - \E(f_\rho)\big).
        \end{align*}
    \end{enumerate}
    \end{thm}
 Theorem \ref{thm:oracle}  indicates that controlling the excess population risk reduces to analyzing the capacity or complexity of the hypothesis space and its approximation capability. 
Moreover, for a fixed hypothesis space, the first two terms vanish as the sample size $n\to \infty$. In this regime, the approximation error becomes the dominant factor determining the overall excess population risk. We note Theorem \ref{thm:oracle} establishes an oracle inequality for a broad class of loss functions, including the hinge loss, least squares loss, and logistic loss. This level of generality makes our results applicable to a wide range of pairwise learning problems (such as ranking, pairwise regression, and metric or similarity learning) across various settings, including kernel methods and neural networks.

    \begin{remark}
        \black{We remark that if the hypothesis space $\calH$ is a VC-class.
    From Lemma \ref{lemma:VC_class} we know $\calH$ satisfies Assumption \ref{assump:cover1}.
    Therefore, part $(a)$ of Theorem \ref{thm:oracle} still holds for any VC-class $\calH$ with $V_1,V_2$ replaced by $2(Pdim(\calH)-1)$, and $s_1,s_2$ be some constants, where $Pdim(\calH)$ is the pseudo-dimension of $\calH$.}
    \end{remark}

       Besides ranking and pairwise least squares regression, our results can also be applied to learning problems whose predictor is independent of the order of input sample pairs, e.g., the metric learning and the similarity learning problems. 
       Let us show how to apply Theorem~\ref{thm:oracle} to metric learning problems. In distance metric learning, we aim to learn a Mahalanobis distance $f(x,x') = (x-x')^\top M (x-x')$, where $M \in \mathbb{S}^{d}_+$ is a positive semi-definite matrix. Let $\tau(y,y')$ be a function of labels such that $\tau(y,y') = 1$ if $y = y'$ and $\tau(y,y')=-1$ else. The performance of $f$ on a sample pair $(z,z')$ is measured by a loss function $\ell\big(\tau(y,y)(f(x,x') - b)\big)$, where $b>0$ is a bias term and $\ell$ is a convex and Lipschitz loss. The hypothesis space here is typically set to be $\calH = \{(x-x')^\top M (x-x'): M \in\mathbb{S}^d_+, \|M\| \le \eta\}$, where $\|\cdot\|$ denotes a regularization norm. Since the Mahalanobis distance and the function $\tau$ are both symmetric, then \eqref{eq:symmetry} holds. In addition, it's easy to verify that Assumptions~$\ref{assump:uniformly_bounded}$ and $\ref{assump:Lipschitz}$ hold by noting that $y$ is assumed to be bounded, and $\beta = 0$ due to the uniform boundedness of the hypothesis space.
       Note $\calH$ is a bounded subset of a $p$-dimensional linear function space (with $p=d(d+1)/2$).
       Then, Assumption~\ref{assump:cover1} holds with $V_1=V_2=d(d+1)/2$.
       Consequently, Theorem~\ref{thm:oracle} yields an oracle inequality for metric learning in which the estimation error term scales as $O(1/\sqrt{n})$. 
\begin{table*}[t]
\centering
\small
\renewcommand{\arraystretch}{1.15}
\setlength{\tabcolsep}{8pt}
\begin{tabular}{l p{4.2cm} l l l}
\toprule
\textbf{Work} &
\textbf{Task} &
\textbf{Loss} &
\textbf{Capacity of $\mathcal{H}$} &
\textbf{Convexity of $\mathcal{H}$} \\
\midrule
\cite{cao2016generalization}
& Metric and similarity learning
& Hinge loss
& VC-class
& Convex \\

\cite{rejchel}
& Ranking
& Convex loss
& VC or non VC-class
& Convex \\

\cite{HZFZ}
& Ranking
& Hinge loss
& VC-class
& Nonconvex \\

Ours
& \makecell[l]{Metric and similarity learning,\\  pairwise least squares regression\\   ranking}
& \makecell[l]{General symmetric \\ loss}
& VC or non VC-class
& Convex or nonconvex \\
\bottomrule
\end{tabular}
\caption{Comparison of different theoretical settings in related works.}
\label{table:setting}
\end{table*}

 \noindent\textbf{Comparison with related works.}
 We give a comparison of our work with the existing results. 
 The most related works are \cite{ranking, rejchel}, which studied the generalization performance of ranking problems. \cite{ranking} derived an oracle inequality of order $O\big((\frac{V\log(n)}{n})^{\frac{1}{2 - \beta}}\big)$ when the $0$-$1$ loss is considered, where $V$ is the VC-dimension of the class of the ranking rules. However, the $0$-$1$ loss is not usually used in practice since the problem considered here is NP-hard and the empirical minimizer $\hat{f}_z$ is not easy to find through an efficient algorithm.
 Instead, we provide an oracle inequality with the same order (part (a) of Theorem~\ref{thm:oracle}) that holds for various common-used surrogate losses including the hinge loss, least squares loss and exponential loss. This makes our results more broadly applicable.
 In addition, \cite{ranking} required that the class of ranking rules is a VC-class, while many hypothesis spaces (e.g., RKHSs \cite{lei2023pairwise}) do not satisfy this assumption. We extend the desired result to a more general setting (part (b) of Theorem~\ref{thm:oracle}, where the hypothesis space is not assumed to be a VC-class), making our results tractable for a wider range of problems.
 \cite{rejchel} established upper bounds for the estimation error with $\beta = 1$ under the assumption that the hypothesis space is convex. We investigate the excess population risk with parameter $\beta\in[0,1]$.
 Moreover, the convexity assumption of the hypothesis space is not required here, which allows our results applicable to study the generalization performance of the neural networks where the hypothesis space is not convex in general. 
\cite{cao2016generalization,jin2009,lei2020sharper} provided the estimation error bounds of order $O(n^{-\frac{1}{2}})$ for the metric and similarity learning problems. As mentioned before, Theorem~\ref{thm:oracle} can also be applied to the metric and similarity learning. 
Very recently, several studies have further advanced the theoretical and algorithmic understanding of pairwise and metric learning \cite{gu2023new,wang2023pairwise,liu2025Trajectory,jia2023adaptive}.
For example,
\cite{wang2023pairwise} investigated pairwise learning with regularization networks and Nyström subsampling, providing algorithmic acceleration but still within shallow models.
More recently, \cite{liu2025Trajectory} established trajectory-dependent generalization bounds for pairwise learning with $\phi$-mixing samples, extending generalization analysis to dependent data. Table~\ref{table:setting} clearly summarizes and contrasts the theoretical settings of our work with those of previous studies, emphasizing the broader applicability of our framework.
 
 \section{Optimal rates with deep ReLU networks}\label{sec:relu}
In this section, we apply our main results to study the generalization performance of pairwise learning problems.
Specifically, we show that pairwise least-squares regression can achieve the optimal excess population risk rate when the hypothesis class is chosen as a family of structured deep ReLU networks that matches the specific form of the true predictor. 
\subsection{A novel approximation of the true predictor}\label{subsec:novel_space}
In this subsection, we first construct a novel structured deep ReLU network as an approximation of the true predictor $f_\rho$. Then, considering the hypothesis space consisting of the networks with this structure, we establish a sharp oracle inequality in the order of $O((\frac{\log(n)}{n})^{\frac{1}{2-\beta}})$ for general anti-symmetric losses. We will show in the next subsection that, the excess risk rate that matches the minimax lower bound for least squares regression \cite{gyorfi,schmidt2020nonparametric} can be obtained when the least squares loss is considered.   

In pairwise learning, we aim to learn a predictor $f:\X\times\X\to\R$ that models the relationship between two samples $x,x'$. 
In many applications such as pairwise ranking, the sign of $f(x,x')$ indicates whether $x$ should be ranked ahead of $x'$. 
Since the prediction depends on the order of the inputs, it is natural to expect $f$ to be anti-symmetric, i.e., $f(x,x')=-f(x',x)$. 
The following assumption ensures that the true predictor $f_\rho$ satisfies this anti-symmetry property.
The hinge loss $\ell(t, y, y') = (1 - sgn(y - y') t)_+$ for ranking and the least squares loss $\ell(t, y, y') = (t - y + y')^2$ for regression satisfy this assumption. 
\begin{assump}\label{assump:anti-sym}
		The loss $\ell$ is anti-symmetric with respect to any predicted values and labels, \ie for any possible predicted value $t \in \R$ and labels $y, y' \in \Y$, there holds
		$\ell(t, y, y') = \ell(-t, y', y).$
	\end{assump}

A good approximation of a model should be one that has the same structure as itself. Hence,
the basic idea of designing an alternative of $f_\rho$ is to construct a model that has the same structure as $f_\rho$. 
The following proposition shows that the true predictor has an anti-symmetric structure under Assumption~\ref{assump:anti-sym}, which suggests the structure of the approximation of $f_\rho$. The proof of Proposition~\ref{prop:anti-symmetric} can be found in Appendix B.
	\begin{prop}\label{prop:anti-symmetric}
		Under Assumption \ref{assump:anti-sym}, the true predictor $f_\rho$ is anti-symmetric, \ie for almost $x, x' \in \X$, there holds
		$f_\rho(x,x') = -f_\rho(x', x). $
		Furthermore, \begin{align}\label{eq:decomp_of_true_predictor}
        f_\rho(x,x') =  \frac{1}{2}f_\rho(x,x') - \frac{1}{2}f_\rho(x',x).
    \end{align}
	\end{prop}
The nice decomposition of $f_\rho$ given in the above proposition indeed tells us how to find a targeted approximation, i.e.,  
we consider making use of a series of ReLU networks to approximate $f_\rho$ with the structure in \eqref{eq:decomp_of_true_predictor}. Before introducing our approximation, we give the definition of deep ReLU neural networks, which are the basis of our structured networks.
We denote $\sigma(t) = (t)_+$ as the ReLU activation function acting componentwise on the vectors. Let $w_0 \in\mathbb{N}^+$ be the dimension of the input space, $w_L=1$ and  $w_l\in\N^+$ be the width of the $l$-th layer for $l=1,\ldots,L-1$, where $L\in\N$ is the depth of the network.
Let $\widetilde{X}\subset\R^{w_0}$.
A deep ReLU neural network $h$ from $\widetilde{\X}$ to $\R$ with depth $L$ has the form
    \begin{align}\label{eq:relu}
     \!h(\widetilde{x}) \!=\! a^{\!\top}\!\sigma(& T_{L- 1} \sigma(T_{L- 2} \cdots \sigma(T_1 (\widetilde{x})^{\!\top} \!\!\!+\! b_1) \cdots \!+\! b_{L-2}) \nonumber\\
        &+ b_{L-1}) + b_L \text{ for any } \widetilde{x} \in \widetilde{\X}, 
     \end{align}
where $T_l\in\R^{w_l\times w_{l-1}}$ indicates the connection matrix between the $l$-th layer and the $ {(l-1)}$-th layer, $b_l \in \R^{w_l}$ is the bias, and $a\in\R^{w_{L-1}}$. Let $\|\cdot\|_0$ denote the number of nonzero elements of the corresponding matrices and vectors. We define the number of nonzero weights and computation units of $h$ by $(\|a\|_0+1) + \sum_{l=1}^{L-1} \|T_l\|_0 + \|b_l\|_0$ and $\sum_{l=1}^{L-1} w_l$, respectively. We say a network $h$ has the complexity $(L,W,U)$ if its depth, the number of nonzero weights and computation units are $L, W$ and $U$.

Now, we can give our structured approximation network $f:\X\times\X\to\R$ as follows
 \begin{align}\label{eq:structured_NNs}
       f(x,x')= g\big(\pi_\eta(h(x,x')),\pi_\eta(h(x',x))\big), 
    \end{align}
for $x,x'\in\X$,
where $\pi_\eta:\R\mapsto \R$ and $g:\R\mapsto \R$ are shallow ReLU networks with fixed complexity, and $h:\X\times \X\mapsto \R$ is a deep ReLU network defined in \eqref{eq:relu} with complexity $(L,W,U)$ and $\widetilde{\X}=\X\times\X$.
The value of $(L,W,U)$ will be properly chosen later for specific problems (see Theorem~\ref{thm:excess} for least squares regression as an example).
Let us give some explanations of \eqref{eq:structured_NNs}.  First, we use a deep ReLU network $h$ to approximate $\frac{1}{2}f_\rho(x,x')$ and $\frac{1}{2}f_\rho(x',x)$ by $h(x,x')$ and $h(x',x)$, respectively. Note that $ \|f_\rho\|_{L^\infty(\X\times\X)}$ is bounded by $\eta$, we hope that the values of $h$ also lie in $[-\eta/2,\eta/2]$. Hence, we consider improving our estimate $h$ by projecting the values of $h$ onto $[-\eta/2,\eta/2]$, i.e., projecting $h$ to $\pi_\eta(h)$ with the projection operator $\pi_\eta$ defined as
    \begin{align*}
        \pi_\eta(t):= \left\{
	\begin{aligned}
		\eta/2, \ \ & t > \eta/2\\
		t, \ \ & t \in[-\eta/2, \eta/2] \\
		-\eta/2, \ \ & t < -\eta/2.
	\end{aligned}\right.
    \end{align*}
The operate $\pi_\eta(h)$ can be written as a shallow ReLU network: $\pi_\eta (t) = \sigma (t) - \sigma (t-\eta/2) - \sigma (-t) + \sigma (-t-\eta/2)$.
Such an expression can be found in \cite{zhou2024optimal}.
With this, two main items $\frac{1}{2}f_\rho(x,x')$ and $\frac{1}{2}f_\rho(x',x)$ appearing in \eqref{eq:decomp_of_true_predictor} are addressed.
It remains to find a ReLU network to handle the difference of the values of $\pi_\eta(h)$ between the sample pair $(x,x')$ and its reverse order pair $(x',x)$.  
By noting that the difference operator can also be represented by a shallow ReLU network $g$, i.e., $x - y = \sigma ([1, -1]\cdot [x, y]) - \sigma ([-1, 1]\cdot [x, y])$, we can give our final approximation $f(x,x') = g\big(\pi_\eta(h(x,x')),\pi_\eta(h(x',x))\big)$.
Figure \ref{graph} gives the specific structure of $f(x,x')$.

  \begin{figure*}[t]
        \centering        \includegraphics[scale=0.45]{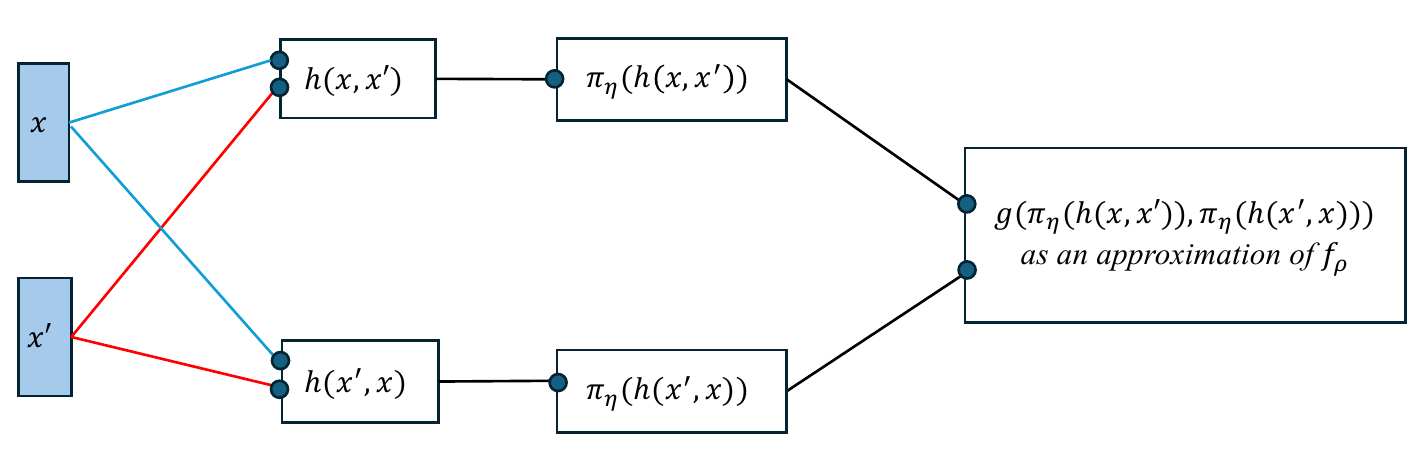}
        \caption{Structure of the designed anti-symmetric deep ReLU network \eqref{eq:structured_NNs} with input $x,x'\in\X$. 
        }\label{graph}
    \end{figure*}

Note that the complexity of a network $f$ of the form \eqref{eq:structured_NNs} can be computed by summing up the corresponding complexities of each sub-network $h,\pi_\eta$ and $g$ directly.
Specifically, the depth of $f$ is the summation of the depth of $h,\pi_\eta$ and $g$. The number of nonzero weights and computation units of $f$ are $2(W_h+W_{\pi_\eta}) + W_g +c$ and $2(U_h+U_{\pi_\eta}) + U_g +c'$ respectively, where $W_\gamma$ and $U_\gamma$ denote the parameters of each sub-network $\gamma\in\{h,\pi_\eta,g\}$, and $c,c'\in\mathbb{N}$ are absolute constants.
The  hypothesis space $\calH$ consists of all possible predictors of the form \eqref{eq:structured_NNs} is defined as 
\begin{align}\label{eq:hypothesis}
        \calH \!=   \calH(&L,W,U)\!:=\!\{\text{$f$ of form \eqref{eq:structured_NNs}: the complexity} \nonumber\\ & \text{of $f$ does not exceed $(L,W,U)$.}\}
    \end{align}
 Here, we use parameters $(L,W,U)\in\N^3$ to measure the capacity (size) of $\calH$. As these parameters increase, the capacity of the hypothesis is getting larger.


From Theorem 7 in \cite{VC} we know the hypothesis space $\calH$ is a VC-class.
Specifically, they proved that $Pdim(\calH) \le CLW\log(U)$.
In the following theorem, we employ the pseudo-dimension to characterize the complexity of the hypothesis space.
Unlike employing the $L^\infty$-covering number, the advantage of using pseudo-dimension is that any boundedness assumptions on the parameters of neural networks are not required.
The proof of Theorem~\ref{thm:oracle_structured_NNs} can be directly derived by part (a) in Theorem~\ref{thm:oracle}.
The detailed proof can be found in Appendix B.  
	    \begin{thm}\label{thm:oracle_structured_NNs}
        Suppose Assumptions \ref{assump:uniformly_bounded}, \ref{assump:Lipschitz} and \ref{assump:anti-sym} hold. Let $\calH$ be the space of structured deep ReLU networks \eqref{eq:hypothesis} and $V = Pdim(\calH)$ be its pseudo-dimension.
        Let $M > 0$ and $\beta\in[0,1]$, suppose that the shifted hypothesis space $\F:=\{\ell(f(x,x'), y, y') - \ell(f_\rho(x,x'), y, y'): f \in \calH\}$ has a variance-expectation bound with parameter pair $(\beta, M)$. Then for any $\delta \in (0,1/2)$, with probability at least $1 - \delta$, 
        \begin{align*}
            \E(\hat{f}_z)& - \E(f_\rho)  \le C_{\eta,K,M,\beta} \Big(\frac{V\log(n)}{n}\Big)^{\frac{1}{2-\beta}}\log(\frac{4}{\delta})\\& + C_{\eta,K} \frac{V}{n}\log^2(\frac{\delta}{2})+ \big(\beta + 2\big) \big(\E(f_\calH) - \E(f_\rho)\big).
        \end{align*}
    \end{thm}

	\subsection{Pairwise least squares regression with deep ReLU networks}\label{subsec:ls_regression}
	In this subsection, we focus on the pairwise least squares regression with loss $\ell(f(x,x'), y, y') = \big(f(x,x') - y + y'\big)^2$.
    We first estimate the approximation error with $\calH$ of the form \eqref{eq:hypothesis}.
    Then, by combining the approximation error bound with Theorem~\ref{thm:oracle_structured_NNs} and choosing the proper capacity of $\calH$, we derive the nearly optimal excess population risk bound for pairwise least squares regression.

    The estimates of the approximation error $\E(f_\calH) - \E(f_\rho)$ are closely related to the true predictor $f_\rho$, which is known to be $f_\rho(x,x') = \tilde{f}_\rho(x) - \tilde{f}_\rho(x')$ for $x,x'\in\X$ \cite{online_pairwise}.
    Here, $\tilde{f}_\rho(x) = \bE[Y|X=x]$ is the regression function. 
    To measure the smoothness of $\tilde{f_\rho}$, we need to introduce the Sobolev space.
    We first assume that $\X = [0,1]^d$ in the remainder of this subsection.
    Let $r \in \N$ be the smoothness index, the Sobolev space $W^{r,\infty}([0,1]^d)$ is defined as the class consisting of functions along with their partial derivatives up to order $r$ lying in $L^\infty([0,1]^d)$.
    The norm in $W^{r,\infty}([0,1]^d)$ is defined as   
    \begin{align*}
        \big\|f\big\|_{W^{r,\infty}([0,1]^d)}:= \max_{\alpha \in \mathbb{Z}_+^d: \|\alpha\|_1 \le r} \big\|D^\alpha f\big\|_{L^\infty([0,1]^d)},
    \end{align*}
    where $\mathbb{Z}_+$ denotes the set of all non-negative integer and $\|\alpha\|_1 = \sum_{i=1}^d |\alpha_i|$ denotes the $l^1$ norm of $\alpha$, $D^\alpha f = \frac{\partial^{\|\alpha\|_1} f}{\partial x_1^{\alpha_1}\cdots\partial x_d^{\alpha_d}}$ denotes the weak partial derivative of $f$ with order $\alpha$.
    
    Assumption \ref{assump:sobolev} assumes all $r$-th partial derivatives of the $\tilde{f}_\rho$ exist and their $L^\infty$ norms are bounded.
\begin{assump}\label{assump:sobolev}
    Suppose for some $r \in \N$, $\|\tilde{f}_\rho\|_{ W^{r,\infty} ([0,1]^{ d} )} \le 1$.   
    \end{assump}
    Since Sobolev norm dominates $L^\infty([0,1]^d)$ norm, Assumption \ref{assump:sobolev} implies Assumption \ref{assump:uniformly_bounded} with $\eta = 2$.

    The following theorem estimates the approximation error, whose proof is given in Appendix B. \begin{thm}\label{thm:approx_error}
        Suppose Assumption \ref{assump:sobolev} holds, and the structured hypothesis space $\calH(L,W,U)$ is defined by \eqref{eq:hypothesis}. Then, for any $\epsilon \in(0,1)$, there exists a deep ReLU network $f$ of form \eqref{eq:structured_NNs} with depth at most $C_{d,r}\log(1/\epsilon)$ and the number of nonzero weights and computation units at most $C_{d,r}\epsilon^{-\frac{d}{r}}\log(1/\epsilon)$ such that
        \begin{align*}
            \big\|f - f_\rho\big\|_{L^\infty([0,1]^{2d})} \le \epsilon.
        \end{align*}
        Further, if we set $W = U = \lceil \exp(L)\rceil$, then the approximation error can be bounded as follows
        \begin{align*}
            \D(\calH) \le C_{d,r}\Big(\frac{L}{\exp(L)}\Big)^{\frac{2r}{d}}.
        \end{align*}
    \end{thm}
Note that Part (b) of Theorem 4 in \cite{yarotsky} shows that, for a deep ReLU network $f$ with depth at most $C_{d,r}\log(\frac{1}{\epsilon})$, achieving an approximation error $\|f-f_\rho\|_{L^\infty} \le \epsilon$ requires \textit{at least}  $C_{d,r}\epsilon^{-\frac{d}{r}}\log^{-3}(\frac{1}{\epsilon})$ nonzero weights. 
Theorem~\ref{thm:approx_error} shows that a deep ReLU network with the same depth can achieve the same approximation error using only $C_{d,r}\epsilon^{-\frac{d}{r}}\log(\frac{1}{\epsilon})$ nonzero weights.
This implies that the bound in Theorem~\ref{thm:approx_error} is optimal up to the logarithmic factor $\log^{4}(\frac{1}{\epsilon})$ in terms of the number of nonzero weights.

    By combining Theorems \ref{thm:oracle_structured_NNs} and \ref{thm:approx_error}, we can establish the following excess population risk bound of order $O(n^{-\frac{2r}{2r+d}})$, which matches the minimax lower bound for the pointwise least squares regression \cite{gyorfi,schmidt2020nonparametric}.
    \begin{thm}\label{thm:excess}
        Suppose Assumption \ref{assump:sobolev} holds, $Prob\{|Y| \le B\} = 1 $ for some constant $B>0$ and the structured hypothesis space $\calH(L,W,U)$ is defined by \eqref{eq:hypothesis}. If we set $W = U = \lceil\exp(L)\rceil$, then for any $\delta\in(0,1/2)$, with probability at least $1 - \delta$,
        \begin{align*}
            \E(\hat{f}_z) - \E(f_\rho) \le & C_{d,r,B}\frac{L^2 \exp(L)\log(n)}{n}\log^2(\frac{4}{\delta})\\& + C_{d,r}\Big(\frac{L}{\exp(L)}\Big)^{\frac{2r}{d}}.
        \end{align*}
        Setting $L = \big\lceil \frac{d}{2r+d} \log(n)\big\rceil$, with probability at least $1- \delta$, there holds
        \begin{align*}
            \E(\hat{f}_z) - \E(f_\rho) \le C_{d,r,B}\log^\tau(n)n^{-\frac{2r}{2r+d}} \log^2(\frac{4}{\delta}),
        \end{align*}
        where $\tau = \max\{3, {2r}/{d}\}$.
    \end{thm}
    \noindent\textbf{Comparison with related works. }
       The minimax lower rate $O\big(n^{-\frac{2r}{2r + s}}\big)$ of the excess risk for pairwise learning is developed in \cite{guo2022distributed} by using the kernel method, where $r\in(1/2,3/2]$ is the smoothness parameter of $f_\rho$, and $s\in(0,1)$ is the effective dimension measuring the capacity of the corresponding RKHS.
       \cite{HZFZ} studied ranking with deep ReLU networks when the hinge loss is considered, they derived the excess risk rate $O\big(n^{-\frac{r(\theta + 1)}{2d + r(\theta+2)}}\big)$, where $r>0$ denotes the smoothness of the target function, and $\theta>0$ is the parameter of noise condition.
       As we mentioned before, the results of previous work \cite{ranking,rejchel} cannot be used for studying the kernel methods or neural networks with the least squares loss, which demonstrates that our general results indeed help us to explore the generalization performance on a variety of problems that cannot be handled by existing approaches.

\section{ {Experiments}}\label{sec:exper}

In this section, we present experimental results to demonstrate the effectiveness of our proposed model in both simulation and real-world data.

\vspace{-3mm}

\subsection{Evaluation of Approximation Error}

This experiment aims to empirically verify the approximation error bounds stated in Theorem~\ref{thm:approx_error}.
The theorem provides upper bounds on the depth and number of nonzero weights required by deep ReLU networks to approximate smooth functions to a prescribed accuracy.
To validate these results, we design a smooth target function and follow the basic idea of \cite{yarotsky, elbrachter2021deep} to explicitly construct a deep ReLU network with controllable depth and sparsity to approximate it.
By measuring the approximation error while increasing the network size (the depth and the number nonzero weights), we examine whether the observed convergence behavior matches the theoretical results.

\noindent\textbf{Target function.} 
We consider a degree-7 B-spline interpolating function $f_\rho$ defined on the interval $[0, 1]$.
Note $f_\rho(x,x') = \tilde{f}_\rho(x) - \tilde{f}_\rho(x')$ in Theorem~\ref{thm:approx_error}. In the following, we focus on approximating the target function in its pointwise form, i.e., we consider $f_\rho(x).$ 
Specifically,
twenty-one equally spaced nodes $\{x_i\}_{i=0}^{20}$ are chosen as $x_i=i/20$, and the corresponding function values $\{y_i\}_{i=0}^{20}$ are drawn independently from a uniform distribution on $[0, 10]$.
The B-spline curve $f_\rho(x)$ is a smooth and piecewise polynomial, which can be uniquely determined by setting $f_\rho(x_i) = y_i$.
That is, $f_\rho(x) = \sum_{p=0}^7 c_{i,p} x^p$ for any $x\in[x_{i-1}, x_{i}]$ and $i=1,\ldots, 20$.
The target function $f_\rho$ is illustrated in Figure~\ref{fig_B-spline}.

\begin{figure}[ht] 
    \centering
\includegraphics[width=0.5\linewidth]{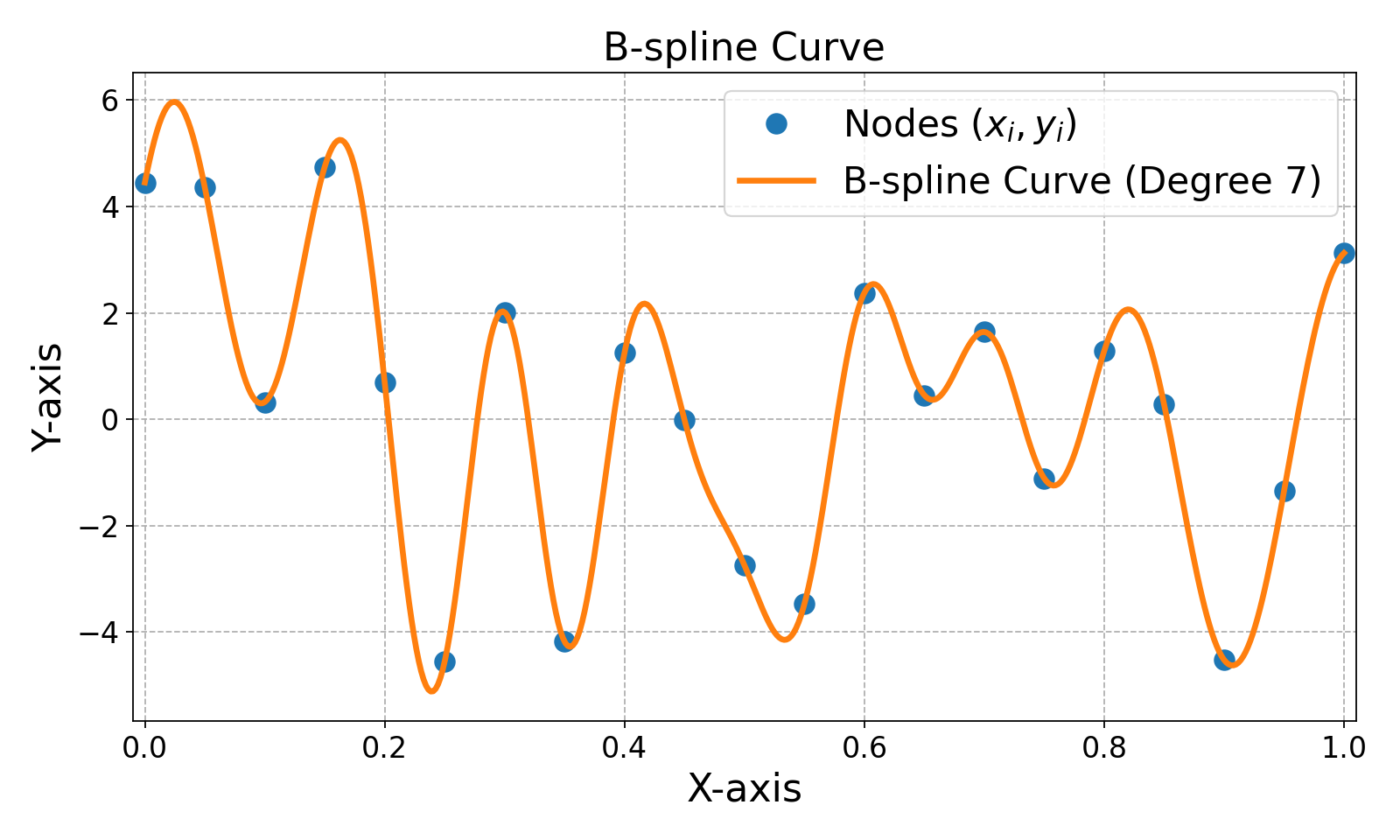}
    \caption{B-spline of order $7$}
    \label{fig_B-spline}
\end{figure}

\noindent\textbf{Network construction.}
We aim to construct a deep neural network $\Psi(x)$ to approximate $f_\rho(x)$.
To do so, we proceed in the following two steps.
Step 1: approximate the power functions $x^p$ by a network $\psi_p(x)$ for all $p=0,\ldots,7$. Step 2: form linear combinations of $\psi_p(x)$ as $\sum_{i=1}^p c_{i,p} \psi_p(x)$ approximate the target polynomial.

\noindent\textbf{Step 1} (\textit{Approximating} $x^p$).  The basic idea to approximate $x^p$ is drawn from \cite{yarotsky, elbrachter2021deep}.  
First, we will find a $L$-layer deep ReLU network \(\psi(x,x')\) that provides a good approximation of the product function \((x,x') \mapsto xx'\). Then, we can approximate $x^2$ by $\psi(x,x)$, i.e., \(\psi(x,x) \approx x^2\), and consequently $x^3$ can be approximated by $\psi(x,\psi(x,x))$. 
Following this idea, we can define a series of networks
\[
\psi_1(x) = x, \qquad 
\psi_p(x) = \psi\bigl(x,\psi_{p-1}(x)\bigr) \quad \text{for } p \ge 2,
\]
and use them to approximate $x^p$ for all $p=0,\ldots,7$. 

Now, we explicitly construct the deep ReLU network \(\psi(x,x')\) to approximate the product function \((x,x')\mapsto xx'\) on \([-1,1]^2\). 
The network $\psi$ with input $[x, x']^\top\in\R^2$ is defined layer by layer as follows: for the first and second hidden layers, we set their weight matrices as \(T_1 \in \mathbb{R}^{4 \times 2}\), \(T_2 \in \mathbb{R}^{6 \times 4}\) and bias \(b_1 \in \mathbb{R}^4\), \(b_2 \in \mathbb{R}^6\):
\[
T_1 = \frac{1}{2}
\begin{bmatrix}
  1 &  1\\
 -1 & -1\\
  1 & -1\\
 -1 &  1
\end{bmatrix}, 
\qquad 
b_1 = \mathbf{0}
\]
and 
\[
T_2 =
\begin{bmatrix}
 1 & 1 & 0 & 0\\
 1 & 1 & 0 & 0\\
 1 & 1 & 0 & 0\\
 0 & 0 & 1 & 1\\
 0 & 0 & 1 & 1\\
 0 & 0 & 1 & 1
\end{bmatrix},
\qquad
b_2 = 
\begin{bmatrix}
 0\\ -2^{-1}\\ 0\\ 0\\ -2^{-1}\\ 0
\end{bmatrix}.
\]
For layers \(\ell = 3, \dots, L+1\), each layer shares the same weight matrix \(T_\ell = T^* \in \mathbb{R}^{6 \times 6}\):
\[
T^* =
\begin{bmatrix}
  0.5 & -1.0 & 0.0 &  0.0 &  0.0 & 0.0\\
  0.5 & -1.0 & 0.0 &  0.0 &  0.0 & 0.0\\
 -0.5 &  1.0 & 1.0 &  0.0 &  0.0 & 0.0\\
  0.0 &  0.0 & 0.0 &  0.5 & -1.0 & 0.0\\
  0.0 &  0.0 & 0.0 &  0.5 & -1.0 & 0.0\\
  0.0 &  0.0 & 0.0 & -0.5 &  1.0 & 1.0
\end{bmatrix}
\]
and different bias \[b_\ell = \begin{bmatrix}
    0& -2^{-2\ell+3}& 0& 0& -2^{-2\ell+3}& 0
\end{bmatrix}^\top\in\R^6.\]

\noindent The output layer has weight \(T_{L+2} \in \mathbb{R}^{1 \times 6}\) and bias \(b_{L+2} \in \mathbb{R}\):
\[
T_{L+2} = \begin{bmatrix} -0.5 & 1.0 & 1.0 & 0.5 & -1.0 & -1.0 \end{bmatrix}, \quad 
  b_{L+2} = 0.\]
Let $\sigma$ denote the ReLU activation function.
The desired network $\psi(x,x')$ has the form  
\begin{align*}
    \psi(x,x')  = & T_{L+2} \sigma\big( T_{L+1} \sigma\big( \cdots \sigma\big( T_1 [x, x']^\top + b_1\big) \cdots\big) + b_{L+2}. 
\end{align*}

\noindent\textbf{Step 2} (\textit{linear combination of} $x^p$).
For each $x$ belongs to subinterval \([x_{i-1}, x_i]\), we can approximate $f_\rho (x)$ by $\Phi_i(x) = \sum_{p=0}^7 c_{i,p}\,\psi_p(x)$.
Then, by concatenating all these subnetworks and denoting the resulting deep ReLU network by \(\Phi(x)\), we obtain a global network that approximates \(f_\rho\) accurately over the entire domain.
Table \ref{tabel_nn} summarize the constructed deep ReLU networks.

\begin{table}[ht] 
\centering
\begin{tabular}{|c|c|c|c|}
\hline
\textbf{Approximator} & \textbf{Target} & \textbf{Depth} & \textbf{Nonzero weights} \\
\hline
\(\psi\)   & \(\!\!(x,x') \mapsto xx'\!\!\)  & \(O(L)\)  & \(O(L)\) \\
\hline
\(\psi_p\) & \( x \mapsto x^p \)     & \(\!O(pL)\!\)  & \(O(pL)\) \\
\hline
\(\Phi\)   & \( f_\rho \)           & \(\!O(pL)\!\)  & \(O(p^2L)\) \\
\hline
\end{tabular}
\caption{Summary of the constructed networks.}\label{tabel_nn}
\end{table}

\begin{figure}[ht] 
    \centering    \includegraphics[width=0.5\linewidth]{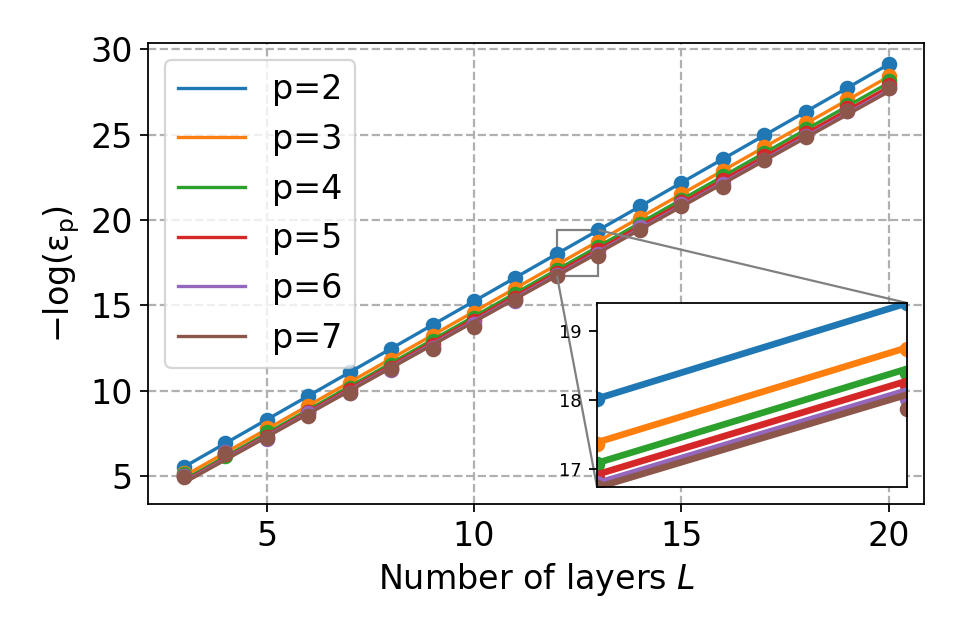}
    \caption{Approximation accuracy $\mathrm{\epsilon}_p = \|\psi_p(x) - x^p\|_\infty$}
    \label{fig_power}
\end{figure}
\begin{figure}[ht]
    \centering    \includegraphics[width=0.5\linewidth]{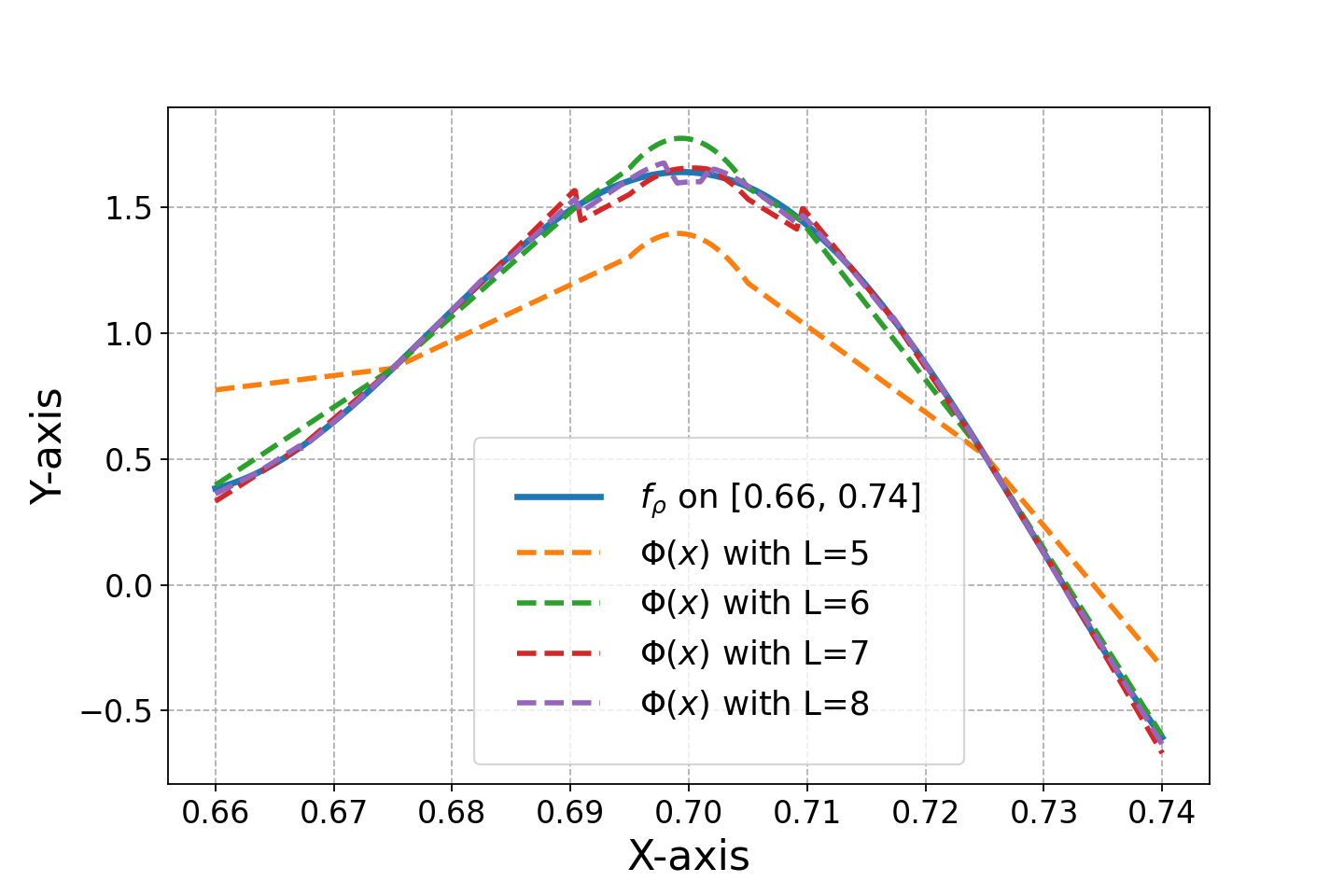}
    \caption{Approximation behavior of the final deep ReLU network $\Phi$ with increasing layers}
   \label{fig_approx_bspline}
\end{figure}

\noindent\textbf{Results.} 
Figure \ref{fig_power} shows how the approximation accuracy $\epsilon_p=\|\psi_p(x) - x^p\|_\infty$  varies as the number of layers $L$ increases for each \(p \in \{2,\ldots,7\}\).
Here, $\epsilon_p$ denotes the infinity-norm between the subnetworks \(\psi_p(x)\) and the power functions \(x^p\). 
Since the infinity norm is taken over all $x\in[0,1]$, we approximate it by taking the maximum absolute error over 20{,}000 uniformly spaced points on \([0,1]\) \cite{quarteroni2006numerical}. 
In Figure \ref{fig_power},  the horizontal axis represents the number of layers \(L\), and the vertical axis shows \(-\log(\epsilon)\).
A larger value of \(-\log(\epsilon)\) indicates  higher  approximation accuracy. 
The plots reveal a striking linear relationship between \(-\log(\epsilon)\) and \(L\), 
indicating that the approximation error decreases \textit{exponentially fast} with respect to the number of layers.
Since the number of layers and nonzero parameters of $\psi_p$ are of the same order (also see Table \ref{tabel_nn}).
Figure \ref{fig_power} implies that we only need $O(\log(1/\epsilon))$ number of layers and nonzero parameters to achieve the approximation accuracy $\epsilon$. 
In addition, functions $x^p$ with smaller degrees \(p\) exhibit better approximation accuracy.

Figure \ref{fig_approx_bspline} illustrates the local approximation behavior of the constructed deep ReLU networks \(\Phi\) for different number of layers \(L = 5, 6, 7, 8\). 
We focus on a representative subinterval \([0.66, 0.74]\). 
The blue solid curve corresponds to the B-spline \(f_\rho\), while the dashed curves represent the network outputs \(\Phi(x)\) with different depths. 
From Figure~\ref{fig_approx_bspline}, we can see that as the the number of layers \(L\) increases, the approximation becomes progressively more accurate. Specifically, for \(L=5\), the network still exhibits visible deviation from the target, whereas for \(L \ge 7\), the curves nearly coincide with \(f_\rho\), capturing both the peak and the shape of $f_\rho$ with high precision. 

\begin{figure}[ht]
    \centering
    
    \hspace{-0.4cm}\subfloat[\footnotesize$-\log(\epsilon)$ vs the number of layers]{\includegraphics[width=0.45\linewidth]{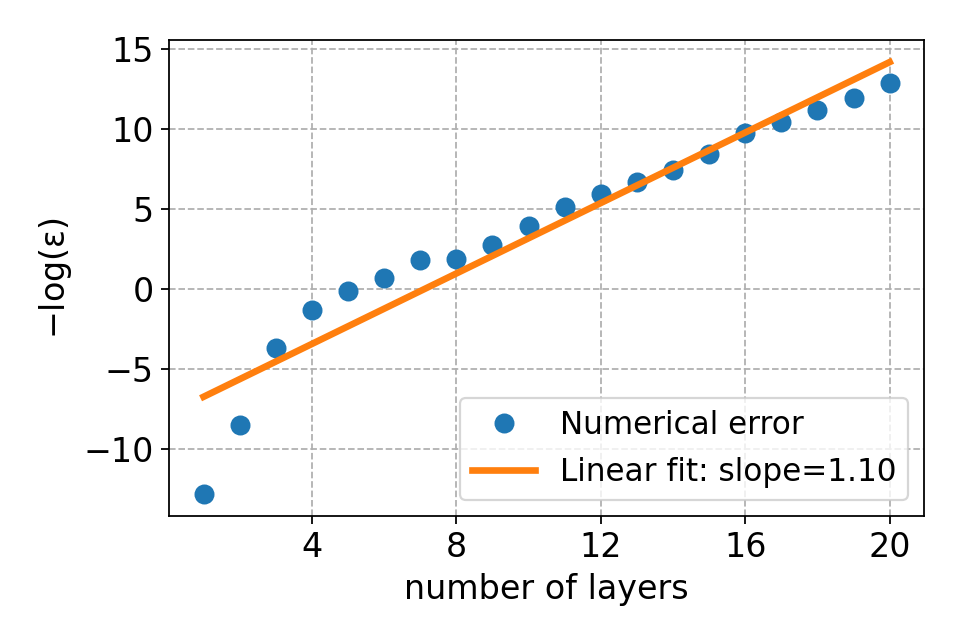}
    \label{fig_layers}}
    \subfloat[\footnotesize $-\log(\epsilon)$ vs the number of nonzero parameters]{\includegraphics[width=0.48\linewidth, height=4.75cm]{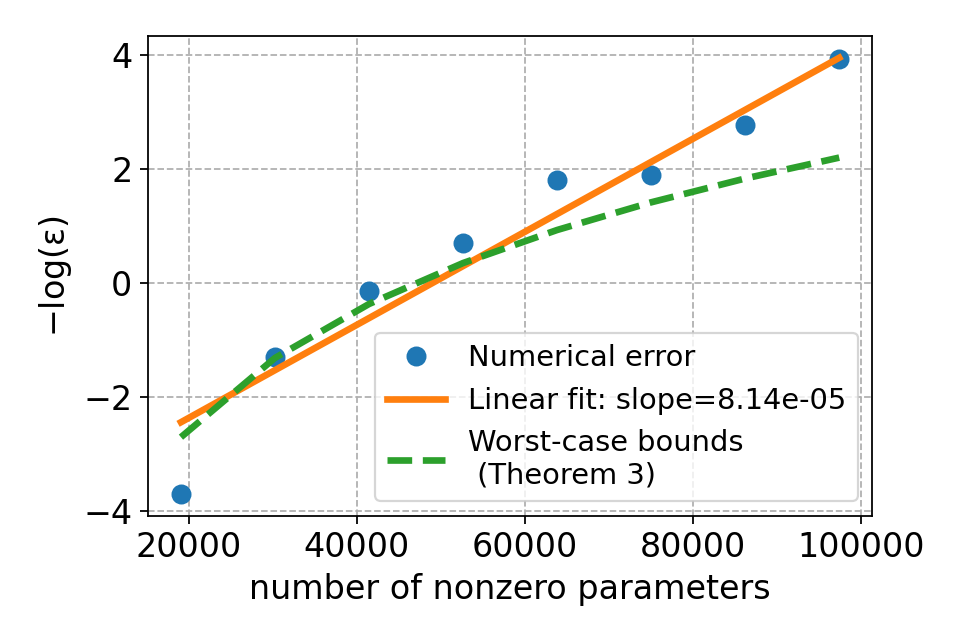}
    \label{fig_params}}
    
    \caption{\small Scaling of approximation error with network size. Here, 
    $\epsilon = \|\Phi(x) - f_\rho(x)\|_\infty$ and the numerical error (blue point) denotes the $-\log(\epsilon)$ as $L$ and $W$ increase.
    The orange line is the linear fit of the numerical errors, showing the approximation error decays exponentially as the number of parameters increases.
    }\label{Fig_networksize}
\end{figure}

Figure~\ref{Fig_networksize} quantitatively verifies the approximation error
bounds in Theorem~\ref{thm:approx_error}.
Theorem~\ref{thm:approx_error} states that, to approximate a target function by deep ReLU network with accuracy $\epsilon$,
one needs to take the number of layers
\(L = O(\log(1/\epsilon))\) and number of nonzero weights
\(W = O(\epsilon^{-c_0})\) \textit{at most}, where \(c_0>0\) is a constant.
Figure~\ref{fig_layers} shows an almost perfectly linear relation between
\(-\log(\epsilon)\) and \(L\), matching the theoretical scaling
\(L \sim \log(1/\epsilon)\).
Figure~\ref{fig_params} also exhibits a linear relation between
\(-\log(\epsilon)\) and \(W\), implying
\(\epsilon \approx \exp(-\alpha W)\) (equivalently,
\(W \approx \alpha^{-1}\log(1/\epsilon)\)), which is smaller than the worst–case requirement \(W \sim \epsilon^{-c_0}\) (green dashed curve).
The gap arises from the fact that our target function $f_\rho$ is fixed,
whereas the worst-case parameter requirements in Theorem~\ref{thm:approx_error} are established for approximating arbitrary $f_\rho$.

\subsection{Non-transitive Pairwise Interactions}  
The aim of this experiment is to show the superiority of our model compared to the well-known RankNet. 
We model the pairwise preference probability between two items \(x, x'\) as 
\[
P_{x,x'} \!=\! P\big(y \!>\! y'\big| x, x'\big) 
\!=\! \sigma\bigl(\!\beta \,[s(x)-s(x')] + \gamma R(x,x')\!\bigr)
\]
where \(\sigma(z) = (1+e^{-z})^{-1}\) is the logistic function. 
The first term \(s(x)-s(x')\) corresponds to the standard score difference used in pairwise ranking models such as RankNet. 
The second term \(R(x,x')\) is a pairwise, anti-symmetric interaction term that cannot be represented as the difference of two univariate functions. 
The hyperparameters \(\beta\) and \(\gamma\) control the relative contributions of these two components. 
Unlike the score difference term, \(R(x,x')\) can be non-transitive, i.e., there exist items \(x_A, x_B, x_C\) such that
\[
R(x_A,x_B) > 0,   R(x_B,x_C) > 0,   \text{ but }   R(x_A,x_C) < 0.
\]
As a result, when \(\gamma\) is relatively large compared to \(\beta\), the model can produce cyclic preferences
\[
P_{x_A,x_B} > \tfrac{1}{2}, \quad 
P_{x_B,x_C} > \tfrac{1}{2}, \quad 
P_{x_A,x_C} < \tfrac{1}{2},
\]
corresponding to \(x_A \succ x_B \succ x_C \succ x_A\).

\noindent{\bf Preference probability setup.}
We set the score function $s(x) = 2\sin(2\pi x)$.
Let $c_1=0.25$, $c_2=0.50$, $c_3=0.75$ and bandwidth $\tau = 0.06$. Define Radial Basis Function (RBF) memberships and the anti-symmetric matrix
\[
\alpha_r(x)=\frac{\exp\!\big(\!-(x-c_r)^2/(2\tau^2)\big)}
{\sum_{i=1}^3\exp\!\big(-(x-c_i)^2/(2\tau^2)\big)},  r\in\{1,2,3\}
\]
and
\[
K=\begin{pmatrix}0&+5&-5\\-5&0&+5\\+5&-1&0\end{pmatrix}.
\]
Then, we define the anti-symmetric interaction term $R(x,x')=\sum_{i,j=1}^3\alpha_i(x)\,K_{ij}\,\alpha_j(x')$.

\noindent{\bf Experiments Setup.}
Inputs $x$ are i.i.d.\ from a mixture on $[0,1]$ (70\% three Gaussians centered at $(c_1,c_2,c_3)$ with std $0.06$, 30\% uniform). Pairs $(x,x')$ are sampled uniformly.
For each pair we draw $U_1,\dots,U_{10}\stackrel{iid}{\sim}\mathrm{Unif}(0,1)$, compute $\bar U=\tfrac{1}{10}\sum_{m=1}^{10}U_m$, and assign the label
\[
y \;=\; \begin{cases}
+1,& \bar U<P_{x,x'},\\
-1,& \text{otherwise}.
\end{cases}
\]
This yields low-variance stochastic labels.
We generate \(80{,}000\) training pairs and \(40{,}000\) testing pairs from the above distribution.
Our proposed model uses the pairwise function \(f(x,x') - f(x',x)\), and RankNet uses the score difference \(g(x) - g(x')\). 
Both the underlying neural functions \(f\) and \(g\) are fully connected three-layer ReLU networks with a hidden width of \(16\). 
We train the models using the Adam optimizer with a batch size of \(4{,}000\) and a learning rate of \(4 \times 10^{-3}\). 
Each model is trained for \(150\) epochs, and all experiments are repeated five times, we report the averaged results over these runs.

\noindent{\bf Results.}
Figure~\ref{figs_rank_pairwise} shows the testing accuracy curves across different \((\beta,\gamma)\) settings.
Specifically, we vary \(\beta \in \{0, 2\}\) and \(\gamma \in \{2, 1, 0.5, 0\}\), where larger \(\beta\) emphasizes the univariate scoring difference component \(s(x)-s(x')\), and larger \(\gamma\) emphasizes the pairwise interaction \(R(x,x')\).
Our results show that when \(\beta = 0.0\) and \(\gamma = 2.0\), RankNet achieves an accuracy close to \(50\%\), 
which is equivalent to random guessing, indicating that it fails to learn any meaningful structure in the absence of score differences.
While our model fits the data well and achieves high accuracy.
When \(\beta=2\) and \(\gamma\) is large, our model again significantly outperforms RankNet, which cannot represent the pairwise interaction structure induced by \(R\). 
As \(\gamma\) decreases, the pairwise effect weakens and the performance gap narrows. 
When \(\gamma = 0\), both models coincide and achieve similar accuracy as expected.

\begin{figure*}[ht]
    \centering
    \begin{minipage}{0.32\linewidth}
        \centering        \includegraphics[width=0.95\linewidth]{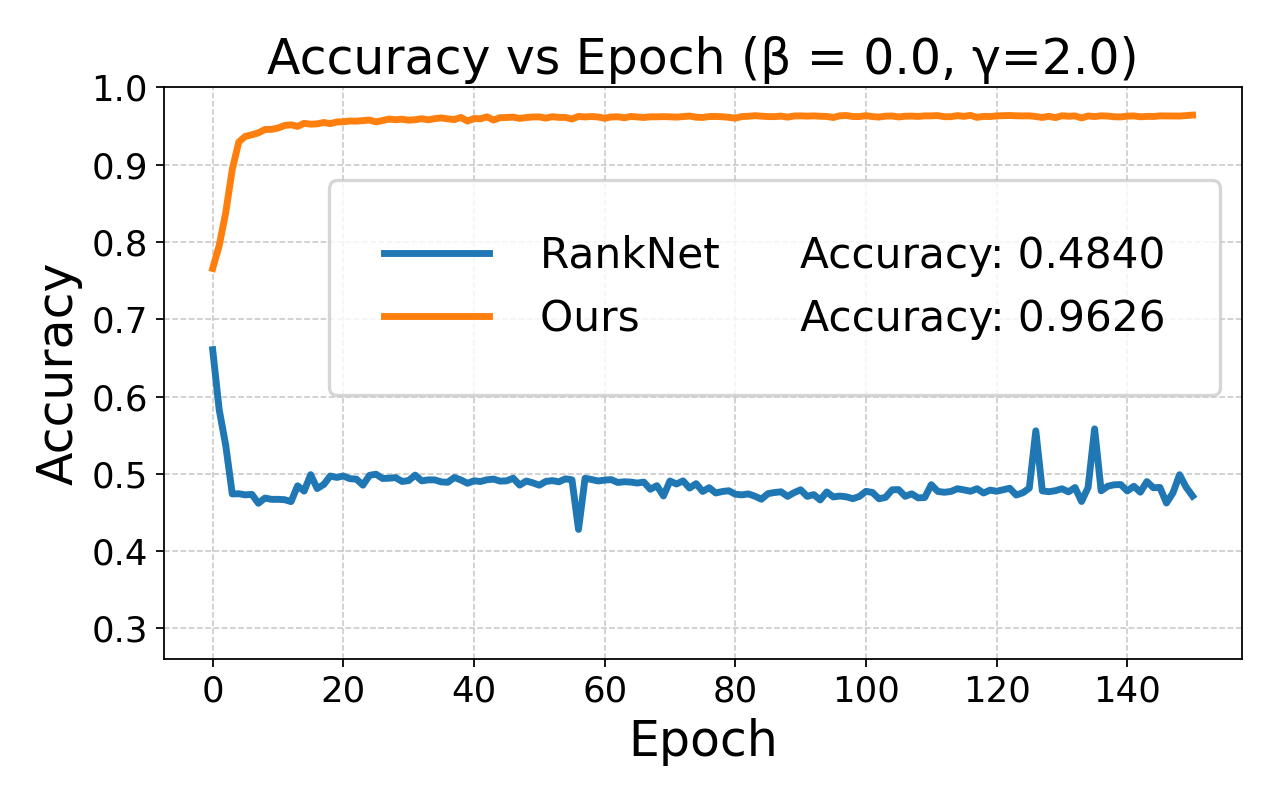}\label{Fig_beta0.0}
    \end{minipage}
    \begin{minipage}{0.32\linewidth}
        \centering        \includegraphics[width=0.95\linewidth]{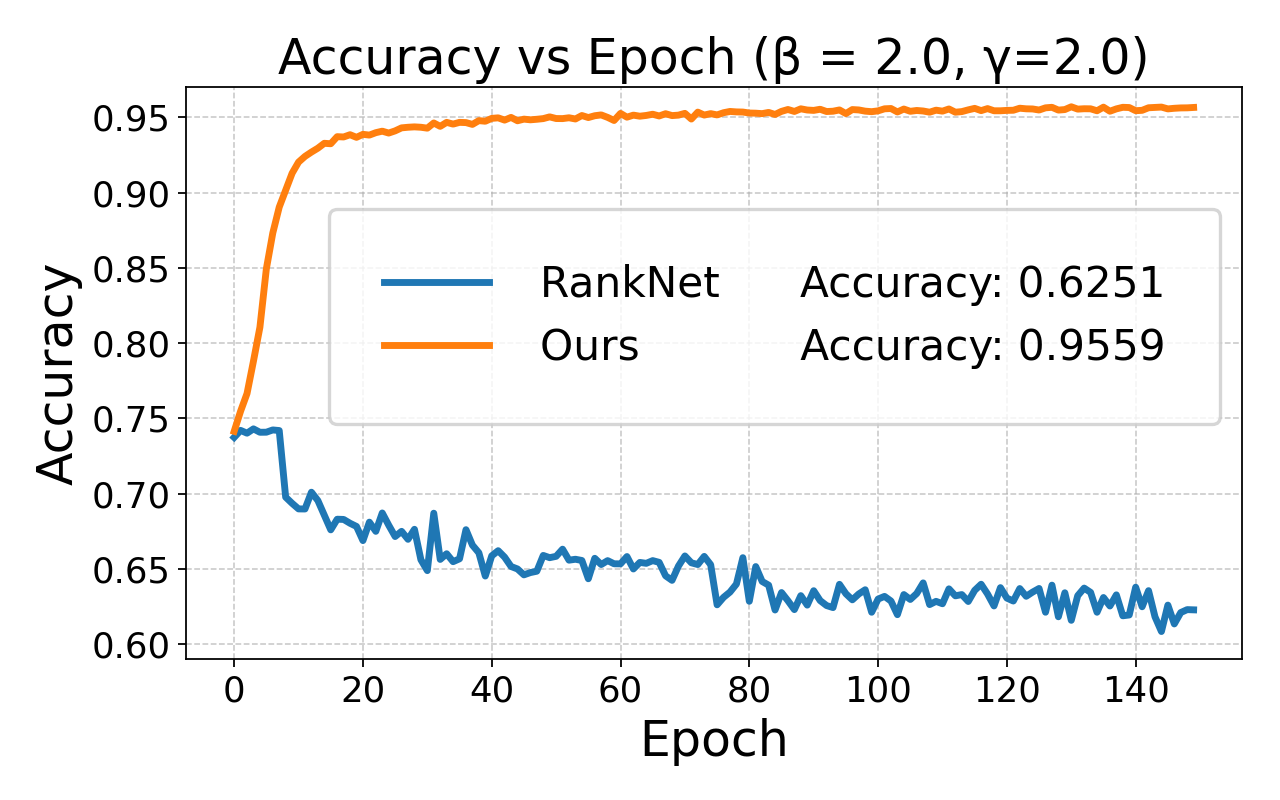}\label{Fig_gamma2.0}
    \end{minipage}
    \begin{minipage}{0.32\linewidth}
        \centering        \includegraphics[width=0.95\linewidth]{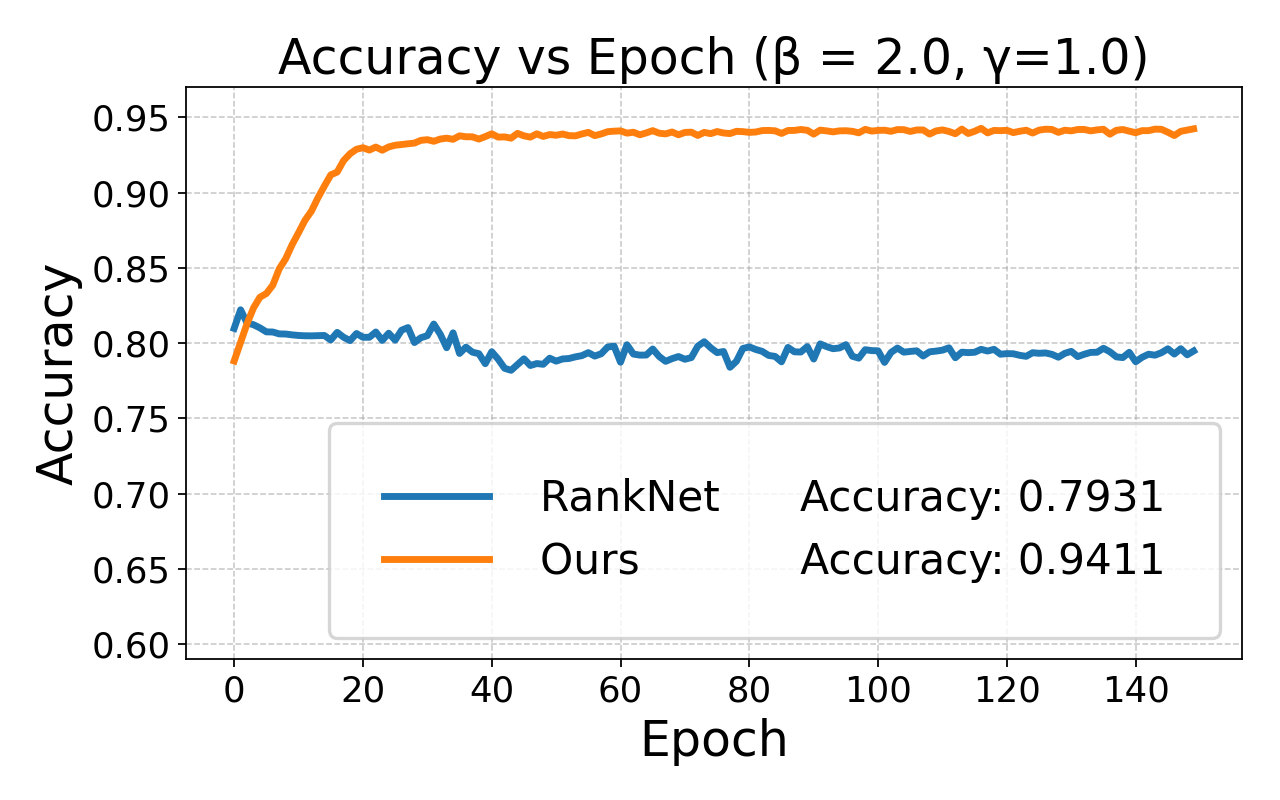}\label{Fig_gamma1.0}
    \end{minipage}
    \begin{minipage}{0.32\linewidth}
        \centering        \includegraphics[width=0.95\linewidth]{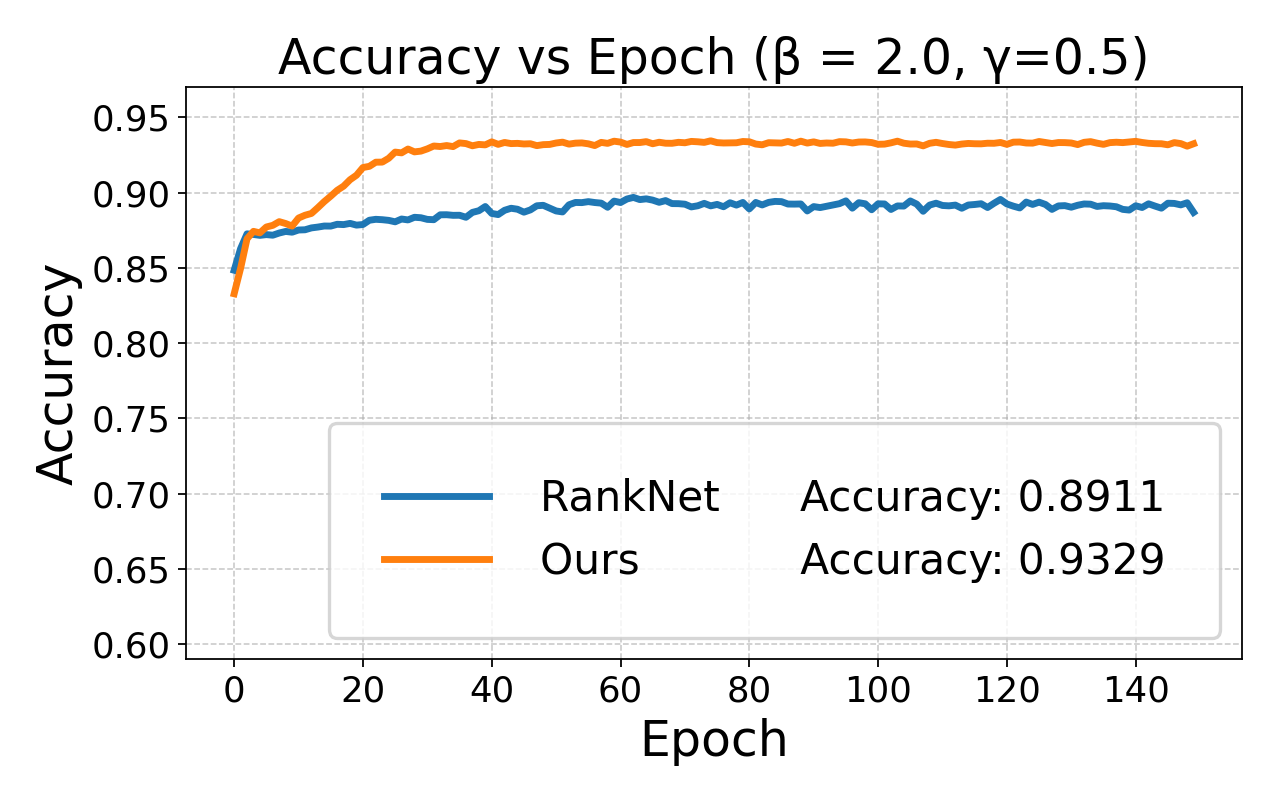} \label{Fig_gamma0.5}
    \end{minipage}
    \begin{minipage}{0.32\linewidth}
        \centering        \includegraphics[width=0.95\linewidth]{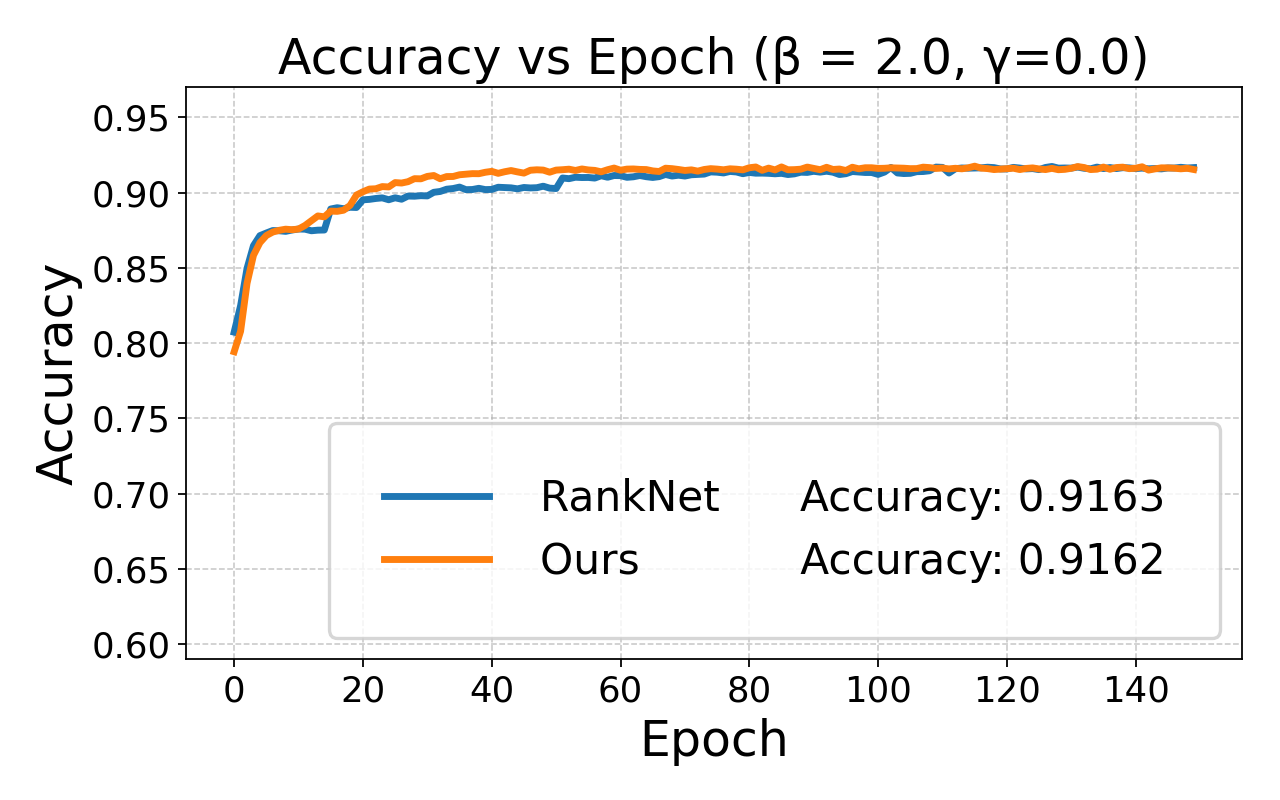}  \label{Fig_gamma0.0}
    \end{minipage}\caption{\small Testing accuracy curves on synthetic data with different settings of \(\beta\) (score strength) and \(\gamma\) (pairwise interaction strength). 
    RankNet fails when the pairwise interaction dominates (\(\gamma = 2\), \(\beta = 0\)) or is strong (\(\gamma=2 \text{ or } 1, \beta=2\)), whereas our model consistently achieves high accuracy. 
    When \(\gamma = 0\), both models coincide.}\label{figs_rank_pairwise}
\end{figure*}

\subsection{Real-world data}
In this experiment, we evaluate our model and the model proposed in  \cite{HZFZ} on the MQ2008 dataset (fold~1, LETOR).
Since Huang’s model is only defined under the hinge loss, we adopt the hinge loss for both models to ensure a fair comparison.

\noindent{\bf Dataset and protocol.}
Each document is represented by 46 features.
Pairs are constructed under the standard pairwise learning-to-rank protocol on the
TRAIN split. We report \emph{test accuracy} per epoch on the TEST split and average
all curves over five independent runs (mean across seeds).

\noindent{\bf Experiments Setup.}
We compare our model
$F(x,x') = f(x,x')-f(x',x)$ with $F(x,x') = h_\xi(f(x, x') - f(x', x))$ proposed in  \cite{HZFZ} using matched capacity.
Here, $h_\xi(t) = \frac{1}{\xi}\big(\sigma (t+\xi) - \sigma(t - \xi) - \xi\big)$ is the one-layer ReLU network used to approximate $\text{sign}(t)$ and $\xi = O\big((\frac{\log^2(n)}{n})^c\big)$, where $n$ is the sample size and $c\in(0, 1)$ is some constant.
Note that the parameter $\xi$ in  \cite{HZFZ}  decreases to $0$ as the sample size increases.
Then, we vary $\xi\in\{0.01, 0.005, 0.001\}$.
Both models employ
a two-hidden-layer neural function $f(x, x')$ with widths $(256,100)$ and ReLU activations, followed by
dropout rates $(0.5, 0.2)$. We optimize the \emph{pairwise hinge loss} with
margin $m=1.0$ using Adam optimizer (learning rate $3\times 10^{-4}$). Mini-batch size is
set to a fixed fraction of the training set (batch\_ratio $=1/20$). Training runs
for 10 epochs, and results are averaged across five repeats.

\noindent{\bf Results.}
Figure~\ref{fig:MQ2008} shows the test accuracy versus training epochs for our model and  \cite{HZFZ}  under different values of the outer ReLU parameter~$\xi$. 
Our model consistently achieves the highest and most stable test accuracy (around~0.84) across all epochs. 
For model proposed in \cite{HZFZ}, the performance gradually degrades as~$\xi$ decreases. 
When~$\xi = 0.01$, the model performs reasonably well, but as~$\xi$ further decreases to~$0.005$ and~$0.001$, both the average accuracy and the stability across epochs noticeably decline. 
This trend aligns with the theoretical observation that the outermost ReLU in \cite{HZFZ} is used to approximate the discontinuous $\mathrm{sign}(x)$ function, where the approximation parameter $\xi = O((\log n / n)^c)$ shrinks to $0$ as the sample size increases. 
Consequently, a smaller~$\xi$ makes the activation more sharply nonlinear and difficult to optimize, leading to unstable training and inferior generalization. 
In contrast, our model avoids this sensitivity to~$\xi$ and maintains superior generalization performance, confirming the robustness and effectiveness of our design.
\begin{figure}[ht] 
    \centering
    \includegraphics[width=0.5\linewidth]{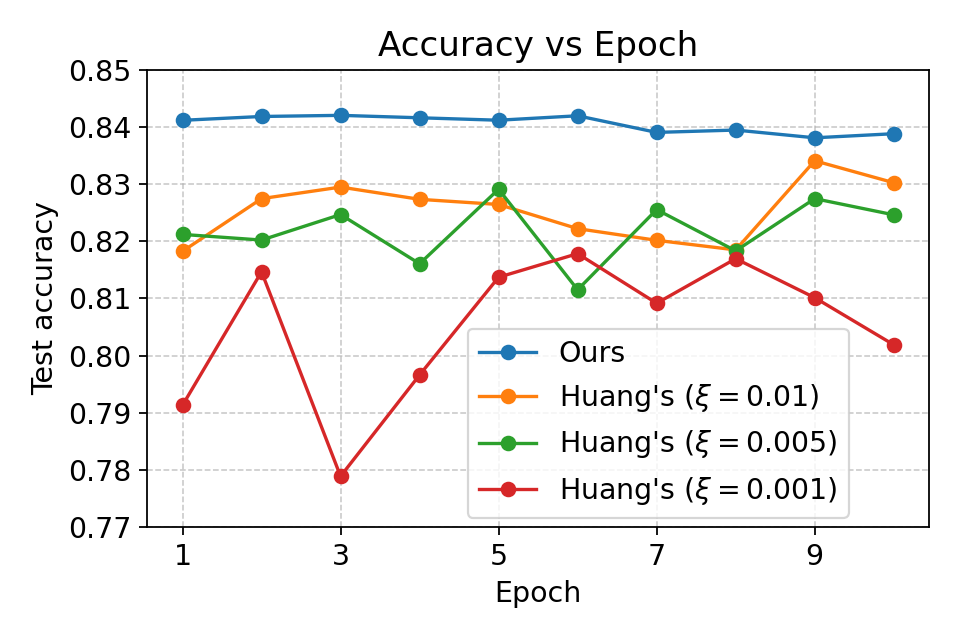}
    \caption{Test accuracy versus epoch for our model and \cite{HZFZ} with different $\xi$ values.}\label{fig:MQ2008}
\end{figure}

    \section{Conclusion}\label{sec:conclusion}
In this paper, we provide generalization analysis of pairwise learning under general settings. Specifically, we establish an oracle inequality of the empirical minimizer in the order of $O((\frac{\log(n)}{n})^{\frac{1}{2-\beta}})$ for a general hypothesis space with the Lipschitz continuous and symmetric loss. Our general results greatly relax the conditions in previous works and can be applied to various learning  problems. As an example, we apply our general results to conduct comprehensive generalization analysis for pairwise learning with structured deep ReLU networks. In particular, the excess population risk bound of order $O(n^{-\frac{2r}{2r+d}})$ that matches the minimax lower bound is achieved for pairwise least squares regression.  Experiments show the effectiveness and efficiency of our proposed methods.

\bibliography{learning}
\bibliographystyle{plain}
\appendix

\begin{center}
   {\Large \bf  Appendix}
\end{center}

\input{submittedsup}

\end{document}

%% file: submittedsup.tex
\section{Proofs for an oracle inequality}\label{sec:proof-main}

	In this section, we present the proof of Theorem 1. The oracle inequality established in Theorem 1 is obtained by deriving the upper bounds of the estimation error $S(\calH) = \E(\hat{f}_z) - \E(f_\calH)$. In classical statistical learning theory with pointwise loss function $\ell:\R\times\Y\to\R$ \cite{cucker,empirical}, the estimation error is often bounded by $2\sup_{f\in\calH}|\bE[\ell(f(X),Y)] - \avg \ell(f(X_i),Y_i)|$, where the analysis of this empirical term heavily depends on the independence of $\{\bE[\ell(f(X),Y)] - \ell(f(X_i),Y_i)\}_{i=1}^n$. However, for pairwise learning with loss $\ell$ and predictor $f\in\calH$, the terms $\{\ell(f(X_i,X_j), Y_i, Y_j)\}_{i\neq j}^n$ in the double-index summation $\E(f) - \E_z(f)$ are dependent, which cannot be handled by the standard techniques in the empirical process theory directly. We consider employing the Hoeffding decomposition \cite{hoeffding} to overcome this dependency difficulty.
 
    Given i.i.d. sample $S = \{Z_i\}_{i=1}^n$ and a symmetric kernel $q:\Z\times\Z\to\R$, an U-statistic $U_n$ associated with kernel $q$ is defined as $U_n = \biavg q(Z_i,Z_j)$. A degenerate U-statistic $W_n$ is an U-statistic such that $\bE[W_n|Z_i] = 0$ for any $i\in\{1,\cdots,n\}$. Hoeffding decomposition breaks $\bE[U_n] - U_n$ into the summation of an i.i.d. term and a degenerate U-statistic term. \ie
	\begin{align*}
		\bE[U_n] - U_n =  2T_n + W_n ,
	\end{align*}
    where
    \begin{align*}
        h(Z_i) &= \bE[U_n] - \bE[q(Z_i, Z) | Z_i],\ 
        T_n = \frac{1}{n} \sum_{i=1}^n h(Z_i),\\
        \hat{h}(Z_i,Z_j) &= \bE[U_n] - h(Z_i) - h(Z_j) - q(Z_i, Z_j),\ 
        W_n = \frac{1}{n(n-1)}\sum_{i\neq j} \hat{h}(Z_i,Z_j).
    \end{align*}
	Note for any $i\in\{1,\cdots,n\}$, there holds $\mathbb{E}[W_n| Z_i] = 0.$ Then $W_n$ is a degenerate U-statistic.

	We now apply Hoeffding decomposition to the estimation error $\E(\hat{f}_z) - \E(f_\calH)$. For any $f \in \calH$, from $\ell(f(x,x'),y,y') = \ell(f(x',x),y',y)$ we know the kernel defined as $q_f(z, z'): = \ell(f(x,x'), y, y') - \ell(f_\rho(x,x'), y, y')$ is symmetric. Further, denote by $U_n^f, h_f, T^f_n, \hat{h}_f, W^f_n$ the corresponding Hoeffding decomposition terms associated with $q_f$. That is, $U_n^f  = \frac{1}{n(n-1)}\sum_{i\neq j}^n q_f(Z_i, Z_j), \ h_f(Z_i) = \mathbb{E}[q_f(Z_i,Z)|Z_i] - \mathbb{E}U_n^f$, $T^f_n = \avg h_f(Z_i)$, $\hat{h}_f(Z_i,Z_j) = \bE[U^f_n] - h_f(Z_i) - h_f(Z_j) - q_f(Z_i,Z_j)$, and $W^f_n = \biavg \hat{h}_f(Z_i,Z_j)$. Then the estimation error can be decomposed as 
	\begin{align}\label{eq:Hoeffding_decomp}
		\E(\hat{f}_z) - \E(f_\calH) 
  & = \E(\hat{f}_z) \!-\! \E_z(\hat{f}_z) \!+\! \E_z(\hat{f}_z) \!-\! \E_z(f_{\calH}) \!+\! \E_z(f_{\calH}) \!-\! \E(f_{\calH})\nonumber\\
        &\le\E(\hat{f}_z) \!-\! \E_z(\hat{f}_z) \!+\! \E_z(f_{\calH}) \!-\! \E(f_{\calH})\nonumber\\
		&= \{\E(\hat{f}_z) \!-\! \E(f_\rho) \!+\! \E_z(f_\rho) \!-\! \E_z(\hat{f}_z)\} \!+\! \{ \E_z(f_{\calH}) \!-\! \E_z(f_\rho) \!+\! \E(f_\rho) \!-\! \E(f_{\calH})\} \nonumber\\
		&= \{EU^{\hat{f}_z}_n - U^{\hat{f}_z}_n\} + \{U^{f_{\calH}}_n - EU^{f_{\calH}}_n\}\nonumber\\
		&= 2T^{\hat{f}_z}_n + W^{\hat{f}_z}_n - 2T^{f_{\calH}}_n - W^{f_{\calH}}_n\nonumber\\
		&= \{2T^{\hat{f}_z}_n - 2T^{f_{\calH}}_n\} + \{W^{\hat{f}_z}_n - W^{f_{\calH}}_n\}\nonumber\\
		&=: S_1(\calH) + S_2(\calH),
	\end{align}
	where in the first inequality we have used the fact that $\E_z(\hat{f}_z) \le \E_z(f_\calH)$. Here, $S_1(\calH)$ consists of i.i.d. terms which can be bounded by using standard techniques in empirical process theory. For bounding the degenerate U-statistics term $S_2(\calH)$, we exploit the decoupling methods in the theory of U-processes \cite{decoupling, ranking}. We estimate $S_1(\calH)$ and $S_2(\calH)$ in the following two subsections, respectively.

\subsection{Upper bounds for $S_1(\calH)$}\label{subsec:S_1}
    Let $\calG = \{g_f(z):= \mathbb{E}\big[\ell(f(x,X), y, Y) - \ell(f_\rho(x,X), y, Y)|X = x, Y = y\big]: f \in \calH\}$ be the function class consisting of the conditional expectation of the symmetric kernels $q_f$. From the definition of $g_f$,  we know $T_n^f = \bE[g_f(Z)] - \avg g_f(Z_i)$ and
    \begin{align}\label{eq:S_1}
        S_1(\calH) = 2\Big( \bE\big[g_{\hat{f}_z}(Z)\big] - \avg g_{\hat{f}_z}(Z_i)  + \avg g_{f_\calH}(Z_i) - \bE\big[g_{f_\calH}(Z)\big]\Big).
    \end{align}
    We estimate $\bE\big[g_{\hat{f}_z}(Z)\big] - \avg g_{\hat{f}_z}(Z_i)$ and $\avg g_{f_\calH}(Z_i) - \bE\big[g_{f_\calH}(Z)\big]$ separately. We will use the Bernstein concentration inequality to control the second term directly. The first term cannot be estimated by the Bernstein concentration inequality due to the appearance of the empirical risk minimizer $\hat{f}_z$ depending on the sample. We will bound it by using the tools (e.g., local complexities and sharp concentration inequalities) in empirical process theory \cite{Bousquet,local}.  

    To guarantee that the local complexity has good properties, we enlarge the function class by introducing the star-shaped class and the star-hull of a function class as follows. 
        \begin{Def}\label{def:star-shaped}
            A function class $\F$ is called a star-shaped class around $0$ if for any $f\in\F$ and $\alpha \in [0,1]$, $\alpha f\in\F$.
        \end{Def}
	\begin{Def}\label{def:star_hull}
		Given a function class $\F$. Denote by
        \begin{align}
            \F^* = \{\alpha f: \alpha \in[0,1], f\in\F\}
        \end{align} the star hull of $\F$ around $0$.
	\end{Def}
    Intuitively speaking, a star-shaped class around $0$ contains all the line segments between $0$ and any point in $\F$. The star-hull of $\F$ is the smallest star-shaped class that contains $\F$.

    Since the star hull $\calG^*$ contains $\calG$, we know the term $\bE\big[g_{\hat{f}_z}(Z)\big] - \avg g_{\hat{f}_z}(Z_i)$ in $S_1(\calH)$ can be bounded by $\sup_{g_f\in\calG^*}\left|\mathbb{E}[g_f] - \avg g_f(Z_i)\right|$. However, this bound is quite loose. Indeed, this upper bound controls the deviation from the generalization errors and empirical errors simultaneously over the whole class $\calG^*$, which might be much more larger than that of the empirical risk minimizer $\bE\big[g_{\hat{f}_z}(Z)\big] - \avg g_{\hat{f}_z}(Z_i)$.  To derive sharp error bounds, we introduce $\calG^*_r := \{g_f \in \calG^*: \mathbb{E}[g_f^2] \le r\}$ with $r>0$, a small subset of $\calG^*$ consisting of predictors satisfying a mild variance condition. The corresponding local complexity is defined as \begin{equation}\label{eq:local_complexity}
		\phi(r) = \mathbb{E}\left[\sup_{g_f \in \calG^*_r} \left|\mathbb{E}[g_f] - \avg g_f(Z_i)\right|\right].
	\end{equation}
   
    We will use the above notion of local complexity and the corresponding fixed point to introduce a sharp concentration inequality. To this end, we need to show that \eqref{eq:local_complexity} is a sub-root function \cite{local}.
	\begin{Def}
		A function $\psi:[0,\infty)\to[0,\infty)$ is sub-root if it is non-negative, nondecreasing and if $r \mapsto \psi(r)/\sqrt{r}$ is nonincreasing for $r > 0$.
	\end{Def}
	Intuitively, sub-root functions increase slowly. From \cite{local} we know they are continuous and have a unique positive fixed point $r^*$. \ie there exists a unique $r^*\in[0,\infty)$ such that $\psi(r^*) = r^*.$ It is clear that \eqref{eq:local_complexity} is non-negative and nondecreasing on $[0, \infty)$. Since $\calG^*$ is star-shaped, we can show $\phi(r)/\sqrt{r}$ is nonincreasing on $[0,\infty)$ by following the arguments in the proof of Lemma 3.4 in \cite{local}. Then, from the above definition we know \eqref{eq:local_complexity} is sub-root and has a unique fixed point $r^*$.
 
    The following lemma from \cite[Theorem 5.4]{Bousquet} establishes a sharp concentration inequality in terms of the fixed point of a sub-root function.
	
	\begin{lemma}[\cite{Bousquet}]\label{lemma:Bousquet}
		Let $M>0$ and $\beta\in[0,1]$, and $\F$ be a star-shaped class around $0$. Suppose $\F$ is uniformly bounded by a constant $b>0$ and has a variance-expectation bound with parameter pair $(\beta, M)$. Let $r^*$ be the unique fixed point of the following sub-root function
		$$\psi(r) = \mathbb{E}\left[\sup_{f \in \F: \bE[f^2] \le r} \left|\avg f(Z_i) - \bE[f]\right|\right].$$
		Then, for any $\delta \in (0,1)$ and $\kappa > 1$, with probability $1- \delta$, we have
		$$\forall f \in \F, \ \mathbb{E}[f] \le \dfrac{\kappa}{\kappa-1} \avg f(Z_i) + C_{\kappa,M,\beta}\left(\left(r^*\right)^{\frac{1}{2-\beta}} + \left(\dfrac{b\log(1/\delta)}{n}\right)^{\frac{1}{2-\beta}} \right).$$
	\end{lemma}
	
    Lemma~\ref{lemma:Bousquet} with $\psi(r)=\phi(r)$ implies that $\bE[g_{\hat{f}_z}(Z)] - \avg g_{\hat{f}_z}(Z_i)\le\frac{1}{\kappa}\bE\big[g_{\hat{f}_z}(Z)\big] + C_{K,M,\beta}\big((r^*)^{\frac{1}{2-\beta}} + \big(\frac{b\log(1/\delta)}{n}\big)^{\frac{1}{2-\beta}} \big)$ with high probability. However, this upper bound that depends on the fixed point $r^*$ is too rough to explore the behavior of $ g_{\hat{f}_z} $. The following proposition introduces the estimates of $r^*$ in terms of the sample size $n$ and the capacity of the hypothesis space, which helps us get a more explicit convergence rate of $\bE[g_{\hat{f}_z}(Z)] - \avg g_{\hat{f}_z}(Z_i)$. 
	
	\begin{prop}\label{prop:fixedpoint}
        Let $r^*$ be the fixed point of the local complexity $(\ref{eq:local_complexity})$. Suppose we are under the same conditions as that in Theorem 1.
      \begin{itemize}
          \item  If the capacity of $\calH$ satisfies Assumption 3, then
		$$r^* \le C_{\eta, K} \max\{V_1,\log(s_1)\}\frac{\log(n)}{n}.$$
  \item    If the capacity of $\calH$ satisfies Assumption 4, then
        $$r^*\le C_{\eta,K}\max\bigg\{\sqrt{s_1'}\Big(\frac{1}{n}\Big)^{\frac{2}{2+V_1'}}, \frac{\log(n)}{n}\bigg\}.$$
      \end{itemize}  
	\end{prop}
\begin{proof}

    Define $\rho_n:= \avg \delta_{Z_i}$ the empirical measure on $\Z$. For any $g_1, g_2 \in \calG^*_r$, we have $\|g_1-g_2\|_{L^2_{\rho_n}}^2= \avg |g_1(Z_i) - g_2(Z_i)|^2$. Let $\{\epsilon_i\}_{i=1}^n$ be the i.i.d. Rademacher variables, \ie $Prob\{\epsilon_i = 1\} = Prob\{\epsilon_i = -1\} = 1/2$. Then, for $r \ge r^*$, according to the standard symmetrization method and the chaining lemma \cite{HDS}, there holds
	\begin{align}\label{ineq:entropy_integral}
		\phi(r) &\le 2\bE\bigg[\bE\bigg[\sup_{g\in\calG_r^*} \Big|\AVG \epsilon_i g(Z_i)\Big| \bigg| Z_1,\cdots,Z_n\bigg]\bigg]\nonumber\\ 
		&\le \dfrac{C}{\sqrt{n}}\mathbb{E}\Big[\int_0^{\sqrt{S}}\sqrt{\log\mathcal{N}(\calG_r^*,L^2_{\rho_n}, t)} dt\Big],
	\end{align}
	where $S:= \sup_{g\in\calG_r^*} \avg g(Z_i)^2$.

 Before proving the two cases, we need to control $\mathcal{N}(\calG_r^*,L^2_{\rho_n}, t)$ and $\bE[S]$. We first upper bound $\mathcal{N}(\calG_r^*,L^2_{\rho_n}, t)$ by the covering number of $\calH$.

    Let $\mathcal{M}$ be a $t$-net of $\calG$ and define $N:= \lceil \frac{A}{t} \rceil$, where $A:= \sup_{g_f\in\calG} \|g_f\|_{L^2_{\rho_n}}$. We claim that $\{\frac{i}{N} g_t: i=1,...,N, g_t \in \mathcal{M}\}$ is a $2t$-net of $\calG^*$. Indeed, for any $g_f^*\in\calG^*$. we know there exists a $g_f\in\calG$ and an $\alpha\in [i/N, (i+1)/N]$ for some $i \in \left\{0,...,N-1\right\}$, such that $g_f^* = \alpha g_f$. Notice that there exists a $g_t \in\M$ such that $\|g_t -g_f\|_{L^2_{\rho_n}} \le t$. Then we conclude that $\|(i/N)g_t - g_f^*\|_{L^2_{\rho_n}} = \|(i/N)g_t - \alpha g_f \|_{L^2_{\rho_n}} \le i/N\|g_t - g_f\|_{L^2_{\rho_n}} + |i/N - \alpha| \|g_f\|_{L^2_{\rho_n}} \le (i/N)t + t \le 2t$. Further, according to the Exercise $4.2.10$ in \cite{HDP} with the fact $\calG^*_r\subset\calG^*$, there holds
	\begin{align}\label{ineq:cover1}
        \mathcal{N}\Big(\calG_r^*, L^2_{\rho_n}, t\Big) \le \mathcal{N}\Big(\calG^*, L^2_{\rho_n}, t/2\Big) \le \mathcal{N}\Big(\calG, L^2_{\rho_n}, t/4\Big) \left\lceil \frac{A}{t}\right\rceil.
	\end{align}
    Now, we estimate the covering number of $\calG$ with respect to metric $L^2_{\rho_n}$. For any $g_{f_1}, g_{f_2} \in \calG$, there holds
	\begin{align*}
		\big\|g_{f_1} - g_{f_2}\big\|_{L^2_{\rho_n}}^2 & = \avg \mathbb{E}\big[\ell(f_1(X_i ,X), Y_i, Y) - \ell(f_2(X_i ,X), Y_i, Y) \big| X_i, Y_i\big]^2 \\
		& \le \avg \mathbb{E}\big[\big(\ell(f_1(X_i ,X), Y_i, Y) - \ell(f_2(X_i, X), Y_i, Y)\big)^2 \big| X_i, Y_i\big] \\
		&\le \frac{K^2}{n}\sum_{i=1}^n\bE\big[\big(f_1(X_i,X) - f_2(X_i,X)\big)^2\big|X_i\big]\\
		&=K^2 \big\|f_1 - f_2\big\|^2_{L^2_{\mu_n\times\rho_{\boldsymbol{x}}}},
	\end{align*}
	where in the first inequality we have used Jensen's inequality for conditional expectation, and in the second inequality we have used Lipschitz property of the loss (Assumption 2). Hence, the above inequality implies
    \begin{align}\label{ineq:cover2}
        \mathcal{N}\Big(\calG,L^2_{\rho_n},t\Big)\le \mathcal{N}\Big(\calH,L^2_{\mu_n\times\rho_{\boldsymbol{x}}},t/K\Big).
    \end{align}
 
    According to Assumptions 1 and 2, we know the constant $A = \sup_{g_f\in\calG} \|g_f\|_{L^2_{\rho_n}} \le  \sup_{g_f\in\calG} \|g_f\|_\infty \le K \sup_{f \in \calH} \|f - f_\rho\|_\infty \le 2K\eta$. Combining \eqref{ineq:cover1} and \eqref{ineq:cover2} together, we finally have \begin{align}\label{ineq:cover3}
        \mathcal{N}\Big(\calG_r^*, L^2_{\rho_n}, t\Big) \le \mathcal{N}\Big(\calH, L^2_{\mu_n\times\rho_{\boldsymbol{x}}}, t/4K\Big) \left\lceil \frac{2K\eta}{t}\right\rceil.
    \end{align}

    Now, we are in a position to derive upper and lower bounds for $\bE[S]$, which will be used in the estimates of the entropy integral later. Recall that $S= \sup_{g\in\calG_r^*} \avg g(Z_i)^2$. Let $\{Z_i'\}_{i=1}^n$ be an independent copy of $\{Z_i\}_{i=1}^n$, the upper bound can be estimated as follows
    \begin{align*}
		\mathbb{E}[S] &= \mathbb{E}\bigg[\sup_{g_f\in\calG_r^*} \avg g_f(Z_i)^2 - \mathbb{E}[g_f^2] + \mathbb{E}[g_f^2]\bigg]\\
		&\le 2\mathbb{E}\bigg[\sup_{g_f\in\calG_r^*} \avg\epsilon_i g_f(Z_i)^2\bigg] + \sup_{g_f\in\calG_r^*} \mathbb{E}[g_f^2] \\
		&\le 8K\eta \mathbb{E}\bigg[\sup_{g_f\in\calG_r^*} \avg \epsilon_i g_f(Z_i)\bigg] + r\\
		&= 8K\eta \mathbb{E}\bigg[\sup_{g_f\in\calG_r^*} \avg \epsilon_i g_f(Z_i) - \epsilon_i \mathbb{E}[g_f] + \epsilon_i \mathbb{E}[g_f]\bigg] + r\\
		&\le 8K\eta \mathbb{E}\bigg[\sup_{g_f\in\calG_r^*} \avg \epsilon_i \big(g_f(Z_i) - g_f(Z_i')\big)\bigg] + 8K\eta\sup_{g_f \in \calG^*_r} \mathbb{E}[g_f]\ \mathbb{E}\bigg[\Big|\avg \epsilon_i\Big|\bigg] + r\\
        &\le 8K\eta \mathbb{E}\bigg[\sup_{g_f\in\calG_r^*} \avg \epsilon_i \big(g_f(Z_i) - g_f(Z_i')\big)\bigg] + 8K\eta \sup_{g_f \in \calG^*_r} \mathbb{E}[g_f^2]^{\frac{1}{2}}\ \mathbb{E}\bigg[\Big|\avg \epsilon_i\Big|^2\bigg]^{\frac{1}{2}} + r\\
        &\le 8K\eta \mathbb{E}\bigg[\sup_{g_f\in\calG_r^*} \avg \epsilon_i \big(g_f(Z_i) - g_f(Z_i')\big)\bigg] + 4K\eta \Big(r + \frac{1}{n}\Big) + r\\
		&=C_{\eta, K} \Big(\mathbb{E}\bigg[\sup_{g_f\in\calG_r^*} \avg g_f(Z_i) - g_f(Z_i')\bigg] + r  +\frac{1}{n} \Big) \le C_{\eta, K}\big(\phi(r) + r + \frac{1}{n}\big),
	\end{align*}
	where in the first and third inequalities we have used the standard symmetrization method (see \cite{HDS}), and in the second inequality we have used the well-known Ledoux-Talagrand contraction principle \cite{ledoux} with $x \mapsto x^2$ and the definition of $\calG_r^*$,
    and in the fourth inequality we have used Jensen's inequality, and in the last second inequality we have used the definition of $\calG_r^*$ and $\sqrt{r/n} \le 1/2(r + 1/n)$, and in the last equality we have used the fact that $\{\epsilon_i\big(g(Z_i) - g(Z_i')\big)\}_{i=1}^n$ and $\{g(Z_i) - g(Z_i')\}_{i=1}^n$ are identically distributed.

    For the lower bound, we can show that there exists a $g_0 \in \calG_r^*$ such that $\bE[g_0^2] = r$ if $r < \sup_{g_f\in\calG} \bE[g_f^2]$. Indeed, if $r < \sup_{g_f\in\calG} \bE[g_f^2]$, by definition we know there exists a $g \in \calG$ such that $\bE[g^2] > r$. Setting $\alpha = \sqrt{\frac{r}{\bE[g^2]}} \in (0,1)$ and $g_0 := \alpha g$, we know $\bE[g^2_0] = r$ and $g_0 \in \calG^*_r$ since $\calG^*_r$ is star shaped. Therefore, we have the following lower bound
	\begin{align*}
		\bE[S] \ge \bE\left[\avg g_0(Z_i)^2\right] = \bE[g_0^2] = r.
	\end{align*}
    Now, we consider the following two cases.
    
    \noindent{\bf First case}: suppose Assumption 3 holds.
    
    According to \eqref{ineq:cover3}, there holds
    \begin{align*}
        \log\mathcal{N}\Big(\calG_r^*,L^2_{\rho_n}, t\Big) \le \log \bigg\{s_1\Big(\frac{1}{t}\Big)^{V_1} \bigg\lceil \frac{2K\eta}{t}\bigg\rceil \bigg\}\le C_{\eta, K} \max\{V_1,\log(s_1)\}\log\Big(\frac{2K\eta}{t}\Big).
    \end{align*}
    Then, \eqref{ineq:entropy_integral} can be bounded as follows
    \begin{align}\label{ineq:phi}
        \phi(r) &\le C_{\eta, K}\sqrt{\frac{\max\{V_1,\log(s_1)\}}{n}}\ \mathbb{E}\bigg[\int_0^{\sqrt{S}}\sqrt{\log\left(\frac{2K\eta}{t}\right)} dt \bigg] \nonumber\\
        &\le C_{\eta, K}\sqrt{\frac{\max\{V_1,\log(s_1)\}}{n}}\ \bE \bigg[\sqrt{2S\log\Big(\frac{C_{\eta,K}}{S}\Big)}\bigg]\nonumber\\
        &\le C_{\eta, K}\sqrt{\frac{\max\{V_1,\log(s_1)\}}{n}}\ \sqrt{\bE[S]\log\bigg(\frac{C_{\eta,K}}{\bE[S]}\bigg)}\nonumber\\
        &\le C_{\eta, K}\sqrt{\frac{\max\{V_1,\log(s_1)\}}{n}}\ \sqrt{\Big(\phi(r) + r + \frac{1}{n}\Big) \log\Big(\frac{C_{\eta,K}}{r}\Big)}.
    \end{align}
    where in the second inequality we have used Lemma 3.8 in \cite{mendelson}, and in the third inequality we have used Jensen's inequality since the function $x\mapsto \sqrt{x\log(C_{\eta,K}/x)}$ is concave, and in the last inequality we have used the upper and lower bounds of $\bE[S]$ obtained as above.
    
    Recall that $r^*$ is a fixed point of $\phi(r)$, \ie $\phi(r^*) = r^*$. Taking the limit $r \to r^*$ and by the continuity of $\phi(r)$ \cite{local}, there holds
	\begin{align*}
		&r^* \le C_{\eta,K} \sqrt{\frac{\max\{V_1,\log(s_1)\}}{n} \Big(r^* + \frac{1}{n}\Big) \log\left(\frac{C_{\eta,K}}{r^*}\right)}\\
		\implies & r^* \le C_{\eta,K} \max\{V_1,\log(s_1)\} \frac{\log(n)}{n}.
	\end{align*}
	The proof of the first case is complete.
 
    \noindent{\bf Second case}: suppose Assumption 4 holds.

    According to \eqref{ineq:cover3}, there holds
    \begin{align*}
        \log\mathcal{N}\Big(\calG_r^*,L^2_{\rho_n}, t\Big) \le s_1'\Big(\frac{1}{t}\Big)^{V_1'} + \log \bigg(\bigg\lceil \frac{2\eta K}{t}\bigg\rceil\bigg)
    \end{align*}
    Then, \eqref{ineq:entropy_integral} can be bounded as follows
    \begin{align*}
        \phi(r) &\le \frac{C}{\sqrt{n}} \bE\Bigg[\int_0^{\sqrt{S}} \sqrt{s_1'\Big(\frac{1}{t}\Big)^{V_1'}} + \sqrt{\log\bigg(\bigg\lceil \frac{2\eta K}{r}\bigg\rceil\bigg)} dt\Bigg]\\
        &\le \frac{C_{\eta,K}}{\sqrt{n}}\Bigg(\bE\Big[\int_0^{\sqrt{S}}\sqrt{s_1'}\Big(\frac{1}{t}\Big)^{\frac{V_1'}{2}}dt\Big] + \sqrt{\Big(\phi(r) + r + \frac{1}{n}\Big) \log\Big(\frac{C_{\eta,K}}{r}\Big)}\Bigg)\\
        &\le \frac{C_{\eta,K}}{\sqrt{n}}\Bigg(\frac{2\sqrt{s_1'}}{2-V_1'}\bE\Big[S^{\frac{2-V_1'}{4}}\Big] + \sqrt{\Big(\phi(r) + r + \frac{1}{n}\Big) \log\Big(\frac{C_{\eta,K}}{r}\Big)}\Bigg)\\
        &\le \frac{C_{\eta,K}}{\sqrt{n}}\Bigg(\frac{2\sqrt{s_1'}}{2-V_1'}\Big(\phi(r) + r + \frac{1}{n}\Big)^{\frac{2-V_1'}{4}} + \sqrt{\Big(\phi(r) + r + \frac{1}{n}\Big) \log\Big(\frac{C_{\eta,K}}{r}\Big)}\Bigg),
    \end{align*}
    where in the first inequality we have used the inequality $\sqrt{a + b} \le \sqrt{a} + \sqrt{b}$ for $a,b\ge0$, and the second inequality follows from \eqref{ineq:phi} directly, and in the last inequality we have used Jensen's inequality for the concave function $x \mapsto (x)^{\frac{2-V_1'}{4}}$ and together with the above upper bound on $\mathbb{E}[S]$.

    Taking the limit $r \to r^*$, there holds
    \begin{align*}
        r^* &\le \sqrt{\frac{C_{\eta,K}}{n}}\Bigg(\frac{\sqrt{s_1'}(r^* + 1/n)^{\frac{2-V_1'}{4}}}{2 - V_1'} + \sqrt{\Big(r^* + \frac{1}{n}\Big)\log\Big(\frac{C_{\eta,K}}{r^*}\Big)}\Bigg)\\
        \implies r^* &\le \sqrt{\frac{C_{\eta,K}}{n}}\max\bigg\{\frac{\sqrt{s_1'}(r^*)^{\frac{2-V_1'}{4}}}{2 - V_1'}, \sqrt{r^*\log\Big(\frac{C_{\eta,K}}{r^*}\Big)}, \sqrt{s_1'\Big(\frac{1}{n}\Big)^{\frac{2-V_1'}{2}} \log\Big(\frac{C_{\eta,K}}{r^*}\Big) }\bigg\}\\
        \implies r^*&\le C_{\eta,K}\max\bigg\{\sqrt{s_1'}(2-V_1')^{-\frac{4}{2+V_1'}} \Big(\frac{1}{n}\Big)^{\frac{2}{2+V_1'}}, \frac{\log(n)}{n}, \sqrt{s_1'\log(n)}\Big(\frac{1}{n}\Big)^{1 - \frac{V_1'}{4}}\bigg\}\\
        \implies r^*&\le C_{\eta,K}\max\bigg\{\sqrt{s_1'} \Big(\frac{1}{n}\Big)^{\frac{2}{2+V_1'}}, \frac{\log(n)}{n}\bigg\},
    \end{align*}
    where in the last inequality we have used the fact $(2-V_1')^{-\frac{4}{2+V_1'}} \le 1$ since $V_1'\in(0,1)$ and the inequality $\frac{\sqrt{\log(n)}}{n^{1-V_1'/4}} \le C\frac{1}{n^{2/(2+V_1')}}$.
The proof of the proposition is complete.
    \end{proof}
    
    Combining Lemma \ref{lemma:Bousquet} and Proposition \ref{prop:fixedpoint} together, and applying Bernstein concentration inequality to $\bE[g_{f_\calH}(Z)] - \avg g_{f_\calH}(Z_i)$, we obtain the following lemma which derives an upper bound for $S_1(\calH)$.

	\begin{lemma} \label{prop:S_1}
            Suppose we are under the same conditions as that in Theorem 1.
            \begin{itemize}
            \item  If the capacity of $\calH$ satisfies Assumption 3, then for any $\delta \in (0, 1/2)$, with probability at least $1 - \delta/2$, there holds
            \begin{align*}
                S_1(\calH) \le C_{\eta, K,M,\beta} \Big(\frac{\max\{\log(s_1),V_1\}\log(n)}{n}\Big)^{\frac{1}{2-\beta}}\log(4/\delta) + \frac{1}{2}\bE\big[g_{\hat{f}_z}(Z)\big] + \frac{\beta}{2} \D(\calH).
            \end{align*}
         \item    If the capacity of $\calH$ satisfies Assumption 4, then for any $\delta\in(0,1/2)$, with probability at least $1 - \delta/2$, there holds
         \begin{align*}
            S_1(\calH) &\le C_{\eta, K,M,\beta} \max\bigg\{ \sqrt{s_1'}\Big(\frac{1}{n}\Big)^{\frac{2}{(2+V_1')(2-\beta)}}, \Big(\frac{\log(n)}{n}\Big)^{\frac{1}{2-\beta}}\log(4/\delta)\bigg\}\\
            &+ \frac{1}{2}\bE\big[g_{\hat{f}_z}(Z)\big] + \frac{\beta}{2} \D(\calH).
         \end{align*}
      \end{itemize}
	\end{lemma}
  \begin{proof}
    Recall that in \eqref{eq:S_1}, we have
	$$S_1(\calH) = 2\bigg(\mathbb{E}\big[g_{\hat{f}_z}\big] - \avg g_{\hat{f}_z}(Z_i) + \avg g_{f_\calH}(Z_i) - \mathbb{E}\big[g_{f_\calH}\big]\bigg).$$
    We first estimate $\avg g_{f_\calH}(Z_i) - \bE[g_{f_\calH}]$. Notice $\|g_{f_\calH} - \mathbb{E}[g_{f_\calH}]\|_{L^\infty(\Z)} \le 2\|g_{f_\calH}\|_{L^\infty(\Z)} \le 2K\|f_\calH - f_\rho\|_{L^\infty(\X)}\le 4K\eta$. Since $g_{\calH}$ is data-independent, we know $g_{\calH}(Z_1),\ldots,g_{\calH}(Z_n)$ are i.i.d. variables. Then, by applying the one-sided Bernstein concentration inequality, for any $\epsilon > 0$, there holds
	$$Prob\bigg\{\avg g_{f_\calH} (Z_i) - \bE[g_{f_\calH}] > \epsilon \bigg\} \le \exp\left\{-\dfrac{n\epsilon^2}{2\big(\mathbb{E}[g_{f_\calH}^2] + \frac{4}{3} K\eta\epsilon\big)}\right\}.$$
    For any $\delta \in (0,1/2)$, suppose $\epsilon = \epsilon^*$ is the solution of the equation $-\frac{n\epsilon^2}{2(\mathbb{E}[g_{\calH}^2] + \frac{4}{3} K\eta\epsilon)} = \log\left(\dfrac{\delta}{4}\right)$. Solving $\epsilon^*$ and plug it into the above probability inequality. Then, we know with probability at least $1 - \delta/4$, there holds
    \begin{align}\label{ineq:Bernstein}
		\avg g_{f_\calH} (Z_i) - \mathbb{E}[g_{\calH}]&\le \frac{\frac{4}{3}K\eta\log(4/\delta) + \sqrt{\left(\frac{4}{3}K\eta\log(4/\delta)\right)^2 + 2n\mathbb{E}[g_{\calH}^2] \log(4/\delta)}}{n}\nonumber\\
            &\le \frac{8K\eta\log(4/\delta)}{3n} + \sqrt{\frac{2\mathbb{E}[g_{\calH}^2] \log(4/\delta)}{n}}\nonumber\\
		&\le \frac{8K\eta\log(4/\delta)}{3n} + \sqrt{\frac{2M\log(4/\delta)}{n}\D(\calH)^\beta}.\nonumber\\
            &\le \frac{8K\eta\log(4/\delta)}{3n} + \frac{2 - \beta}{2} \Big(\frac{2M\log(4/\delta)}{n}\Big)^{\frac{1}{2- \beta}} + \frac{\beta}{2} \D(\calH)\nonumber\\
            &\le C_{\eta,K,M,\beta}\Big(\frac{1}{n}\Big)^{\frac{1}{2-\beta}}\log(4/\delta) + \frac{\beta}{2} \D(\calH),
	\end{align}
    where $\D(\calH) = \bE[g_{f_\calH}] = \inf_{f\in\calH}\E(f) - \E(f_\rho)$ is the approximation error, in the second inequality we have used the inequality $\sqrt{a + b} \le \sqrt{a} + \sqrt{b}$ for $a,b\ge0$, in the third inequality we have used Jensen's inequality for conditional expectation and the variance-expectation bound for shifted hypothesis space. Indeed, there holds $\mathbb{E}[g_{\calH}^2] \le M\left(\mathbb{E}[g_{\calH}]\right)^\beta = M\D(\calH)^\beta$. In the last second inequality we have used Young's inequality $ab \le \dfrac{1}{p} a^p + \dfrac{1}{q} b^q$ with $a = \sqrt{\frac{2M\log(4/\delta)}{n}}, b = \sqrt{\D(\calH)^\beta}, p = \frac{2}{2 - \beta}$, and $q = \frac{2}{\beta}$, in the last inequality we have used the fact $\big(\log(4/\delta)\big)^{1/(2-\beta)} \le \log(4/\delta)$ since $\frac{1}{2-\beta} \le 1$ and $\log(4/\delta) > 1$.

    Now we estimate $\mathbb{E}\big[g_{\hat{f}_z}\big] - \avg g_{\hat{f}_z}(Z_i)$. Notice $g_{\hat{f}_z} \in \calG \subset \calG^*$ and $\|g_{f} - \mathbb{E}[g_{f}]\|_{L^\infty(\Z)} \le 4K\eta$ for any $g_f\in\calG^*$. Applying Lemma \ref{lemma:Bousquet} with $\F = \calG^*$, $b = 4K\eta$ and $\kappa = 4$, we know with probability at least $1 - \delta/4$, there holds
    \begin{align*}
        \bE\big[g_{\hat{f}_z}(Z)\big] - \avg g_{\hat{f}_z}(Z_i)\le\frac{1}{4}\bE\big[g_{\hat{f}_z}(Z)\big] + C_{\eta,K,M,\beta}\Big(\big(r^*\big)^{\frac{1}{2-\beta}} + \Big(\frac{\log(4/\delta)}{n}\Big)^{\frac{1}{2-\beta}} \Big)
    \end{align*}
    Combining the above inequality with \eqref{ineq:Bernstein} together, we know with probability at least $1 -\delta/2$, there holds
    \begin{align*}
        S_1(\calH) \le C_{\eta,K,M,\beta}\bigg(\big(r^*\big)^{\frac{1}{2-\beta}} + \Big(\frac{1}{n}\Big)^{\frac{1}{2-\beta}}\log(4/\delta)\bigg) + \frac{1}{2}\bE\big[g_{\hat{f}_z}(Z)\big] + \frac{\beta}{2} \D(\calH).
    \end{align*}
    {\bf First case}: Suppose Assumption 3 holds.
    
    According to Proposition \ref{prop:fixedpoint}, with probability at least $1 - \delta/2$, there holds
    \begin{align}
        S_1(\calH) \le C_{\eta,K,M,\beta} \Big(\frac{\max\{V_1,\log(s_1)\}\log(n)}{n}\Big)^{\frac{1}{2-\beta}}\log(4/\delta) + \frac{1}{2}\bE\big[g_{\hat{f}_z}(Z)\big] + \frac{\beta}{2} \D(\calH).
    \end{align}
    {\bf Second case}: Suppose Assumption 4 holds.

    According to Proposition \ref{prop:fixedpoint}, with probability at least $1 - \delta/2$, there holds
    \begin{align*}
        S_1(\calH) &\le C_{\eta,K,M,\beta} \max\bigg\{\sqrt{s_1'}\Big(\frac{1}{n}\Big)^{\frac{2}{(2+V_1')(2-\beta)}}, \Big(\frac{\log(n)}{n}\Big)^{\frac{1}{2-\beta}}\log(4/\delta)\bigg\}\\
            &\ \ \ \ + \frac{1}{2}\bE\big[g_{\hat{f}_z}(Z)\big] + \frac{\beta}{2} \D(\calH).
    \end{align*}
 The proof of the lemma is complete.
 \end{proof}
    \subsection{Upper bounds for $S_2(\calH)$}\label{subsec:S_2}
    We begin by introducing some notations that will be used frequently later. Let $\Q:=\{h(z,z'): h(z,z') = h(z',z)$ for almost $z,z'\in\Z\}$ be some class of symmetric functions from $\Z\times\Z$ to $\R$. Denote a empirical probability measure on $\Z\times\Z$ by $\xi:= \frac{1}{n(n-1)} \sum^n_{i\neq j} \delta_{(Z_i,Z_j)}$. For any $h\in\Q$, we define the $L^2_\xi$ norm of $h$ by
    \begin{align*}
        \|h\|_{L^2_\xi} = \bigg(\biavg |h(Z_i,Z_j)|\bigg)^{\frac{1}{2}}.
    \end{align*}
    Further, we define $\W:=\{\hat{h}_f: f\in\calH\}$ as the class consisting of the kernel of the degenerate U-statistics.
    
    We introduce a probability bound of the degenerate U-statistic, which controls the degenerate U-statistic in terms of the expectations of the corresponding Rademacher chaos and Rademacher complexities.
	\begin{lemma}[\cite{ranking}]\label{lemma:ranking}
		Define the supremum of a degenerate U-statistic over the hypothesis space $\calH$ as
		$$Z = \sup_{f \in \calH}\left|\sum_{i\neq j} \hat{h}_f(Z_i,Z_j)\right|,$$
        where $\hat{h}_f(Z_i,Z_j) = \bE[U_n^f] - h_f(Z_i) - h_f(Z_j) - q_f(Z_i,Z_j)$.
		Then there exists an absolute constant $C>0$ such that for all $n$ and $t>0$,
		$$Prob\big\{Z > C\mathbb{E}[Z_\epsilon] + t \big\} \le \exp\left(-\dfrac{1}{C}\min\left( \left(\dfrac{t}{\mathbb{E}[U_\epsilon]}\right)^2, \dfrac{t}{\mathbb{E}[M] + Fn}, \left(\dfrac{t}{F\sqrt{n}}\right)^{2/3}, \sqrt{\dfrac{t}{F}}\right)\right),$$
		where $\epsilon_1,...,\epsilon_n$ are i.i.d. Rademacher variables and  $ 
		 M = \sup_{f \in \calH, k = 1,...,n} \bigg|\sum_{i=1}^n \epsilon_i \hat{h}_f (Z_i, Z_k)\bigg|,  F = \sup_{f \in \calH} \|\hat{h}_f\|_{L^\infty(\Z\times\Z)},$ 
		$Z_\epsilon = \sup_{f \in \calH} \bigg| \sum_{i\neq j}^n\epsilon_i \epsilon_j\hat{h}_f(Z_i, Z_j)\bigg|,  U_\epsilon = \sup_{f \in \calH} \sup_{\alpha: \|\alpha\|_2 \le 1} \sum_{i \neq j}^n \epsilon_i \alpha_j \hat{h}_f (Z_i, Z_j).$
	\end{lemma}
	
	Notice that 
	$$S_2(\calH) = - W^{\hat{f}_z}_n + W^{f_{\calH}}_n \le \frac{2}{n(n-1)} Z.$$
	For any $\delta \in (0,1/2)$, we set the equation
	$$\exp\left(-\frac{1}{C}\min\left( \left(\frac{t}{\mathbb{E}[U_\epsilon]}\right)^2, \frac{t}{\mathbb{E}[M] + Fn}, \left(\frac{t}{F\sqrt{n}}\right)^{2/3}, \sqrt{\frac{t}{F}}\right)\right) = \delta/2,$$
        and solve it for variable $t$, then by Lemma \ref{lemma:ranking} we know with probability at least $1 - \delta/2$, there holds
	$$S_2(\calH) \le \frac{C\log^2(2/\delta)}{n^2} \Big(\mathbb{E}[Z_\epsilon] + \mathbb{E}[U_\epsilon] + \mathbb{E}[M] + Fn \Big).$$
    Therefore, to derive upper bounds for $S_2(\calH)$, we suffice to estimate the Rademchaer chaos $\bE[Z_\epsilon]$ and the Rademacher complexities $\bE[U_\epsilon]$ and $\bE[M]$.
	
    We first estimate $\mathbb{E}[Z_{\epsilon}]$. Since the Rademacher chaos $Z_\epsilon = \sup_{f\in\calH}\big|\sum^n_{i\neq j}\epsilon_i\epsilon_j\hat{h}_f(Z_i,Z_j)\big|$ conditioned on sample $S$ is not sub-Gaussian with respect to the metric on $\calH$ induced by any empirical norm of $\hat{h}_f$, then we cannot apply the classical chaining methods directly \cite{HDP,HDS}. Instead, we introduce two lemmas that estimate $\bE[Z_\epsilon]$ directly. The first lemma from \cite{ying_chaos} established a maximal inequality of $Z_\epsilon$. The second lemma is very similar to Theorem 2 in \cite{ying_chaos}, which controls $\bE[Z_\epsilon]$ in terms of an entropy integral.
	
	\begin{lemma}[\cite{ying_chaos}]\label{lemma:maximal_ineq}
        Let $\{h_1,...,h_N\} \subset \Q$ be a finite class of functions contained in $\Q$, and $\{\epsilon_1,\ldots,\epsilon_N\}$ be i.i.d. Rademacher variables, then there holds
		$$\mathbb{E}\bigg[\max_{k \in\{1,...,N\}}\Big|\biavg \epsilon_i\epsilon_j h_k(Z_i,Z_j)\Big| \bigg| Z_1,\ldots,Z_n\bigg] \le 2\sqrt{2}en\log(N + 1) \max_{k \in\{1,...,N\}}\|h_k\|_{L^2_\xi}.$$
	\end{lemma}

    With the above maximal inequality, we can bound the Rademacher chaos in terms of an entropy integral by using the same arguments in classical chaining methods \cite{HDP,HDS}.
	
    \begin{lemma}\label{lemma:chaos}
        Let $\{\epsilon_1,\ldots,\epsilon_N\}$ be i.i.d. Rademacher variables, then there holds
        $$\mathbb{E}\bigg[\sup_{h\in\Q} \biavg \epsilon_i\epsilon_jh(Z_i,Z_j)\bigg|Z_1,\ldots,Z_n\bigg] \le n \Big(72\sqrt{2}e\int_{0}^{D/2} \log\Big(\mathcal{N}\big(\Q, L^2_\xi, t\big)\Big) \, dt + D \Big),$$
        where $D = \sup_{h\in\Q} \|h\|_{L^2_\xi}$ (Note that $D = D(Z_1,\ldots,Z_n)$ depends on the sample).
	\end{lemma}
  \begin{proof} For each $k \in \N$, let $\mathcal{N}_k$ be a $\frac{D}{2^k}$-net of $\Q$ with respect to the metric $L^2_\xi$. From the definition of the net (see Definition 1), we can define a map $\pi_k$ from $\Q$ to $\mathcal{N}_k$ such that $\big\|h - \pi_k(h)\big\|_{L^2_\xi} \le \frac{D}{2^k}$ for all $h\in \Q$. Then $\sup_{h \in \Q} \|h - \pi_k(h)\|_{L^2_\xi} \le \frac{D}{2^k}\to 0$ as $k \to \infty$. Further, notice that $\sup_{h \in \Q} \big|\sum_{i\neq j}^n \epsilon_i\epsilon_j \big(h - \pi_k(h)\big)(Z_i,Z_j)\big|$ is uniformly bounded for any $k$. Then, by Dominated Convergence Theorem, we know $\mathbb{E}\big[\sup_{h \in \Q} \big|\sum_{i\neq j} \epsilon_i\epsilon_j \big(h - \pi_k(h)\big)(Z_i,Z_j)\big| Z_1,\ldots,Z_n\big] \to 0$ as $k\to \infty$.
	
    Take an arbitrary $h_0 \in \Q$, we can set $\mathcal{N}_0 = \{h_0\}$ since $\pi_0(h) = h_0$ for any $h \in \Q$. Then, conditioned on the sample $\{Z_1,\ldots,Z_n\}$, there holds
	\begin{align}
            \mathbb{E}\bigg[\sup_{h \in \Q} \Big|\sum_{i\neq j}^n \epsilon_i\epsilon_j h(Z_i,Z_j)\Big|\bigg] &\le \mathbb{E}\bigg[\sup_{h \in \Q} \Big|\sum_{i\neq j}^n \epsilon_i\epsilon_j \big(h- h_0\big)(Z_i,Z_j) \Big|\bigg] + \mathbb{E}\bigg[\Big|\sum_{i\neq j}^n \epsilon_i\epsilon_j h_0(Z_i,Z_j)\Big|\bigg]\nonumber\\
            &\le \mathbb{E}\bigg[\sup_{h \in \Q} \Big|\sum_{i\neq j}^n \epsilon_i\epsilon_j \big(h- h_0\big)(Z_i,Z_j) \Big|\bigg] + \bigg( \mathbb{E} \bigg[\Big|\sum_{i\neq j}\epsilon_i\epsilon_j h_0(Z_i,Z_j)\Big|^2\bigg]\bigg)^{1/2} \nonumber\\
            &= \mathbb{E}\bigg[\sup_{h \in \Q} \Big|\sum_{i\neq j}^n \epsilon_i\epsilon_j \big(h- h_0\big)(Z_i,Z_j) \Big|\bigg] + \sqrt{n(n-1)}\big\|h_0\big\|_{L^2_\xi} \nonumber\\
            &\le \mathbb{E}\bigg[\sup_{h \in \Q} \Big|\sum_{i\neq j}^n \epsilon_i\epsilon_j \big(h - \pi_k(h)\big)(Z_i,Z_j)\Big|\bigg] \nonumber\\
            & \ \ \ +\sum_{m=1}^{k} \mathbb{E}\bigg[\sup_{h \in \Q} \Big|\sum_{i\neq j}^n \epsilon_i\epsilon_j \big(\pi_m(h) - \pi_{m-1}(h)\big)(Z_i,Z_j)\Big|\bigg] +  nD\nonumber\\
		&\to \sum_{k=1}^{\infty} \mathbb{E}\Bigg[\sup_{\substack{
				h \in \Q, \\
				(\pi_k(h) , \pi_{k-1}(h)) \in \\
				\mathcal{N}_k\times \mathcal{N}_{k-1} 
		}} \Big|\sum_{i\neq j}^n \epsilon_i\epsilon_j \big(\pi_k(h) - \pi_{k-1}(h)\big)(Z_i,Z_j)\Big|\Bigg] + nD \label{ineq:chaining}
	\end{align}
    as $k\to\infty$, where in the second inequality we have used Jensen's inequality, and in the first equality we just expand the term $\big|\sum_{i\neq j}^n\epsilon_i\epsilon_j h_0\big|^2$ and compute its expectation, and in the last step we have used Dominated Convergence Theorem.
    
    Denote by $|A|$ the cardinality of a set $A$, we know $\big|\mathcal{N}_k\times \mathcal{N}_{k-1}\big| = \big|\mathcal{N}_k\big| \big|\mathcal{N}_{k-1}\big| \le \mathcal{N}\big(\Q, L^2_\xi, \frac{D}{2^k}\big)^2$. Notice that for any $h \in \Q$ and $k\in\N^+$, $\big\|\pi_k(h) - \pi_{k-1}(h)\big\|_{L^2_\xi} \le \big\|\pi_k(h) - h\big\|_{L^2_\xi} + \big\|h - \pi_{k-1}(h)\big\|_{L^2_\xi} \le \frac{3}{2^{k-1}}D$. Then. by Lemma \ref{lemma:maximal_ineq}, \eqref{ineq:chaining} can be bounded as follows
	\begin{align*}
        \bE\bigg[\sup_{h\in\Q}\Big|\sum_{i\neq j}^n\epsilon_i\epsilon_j h\Big|\bigg| Z_1,\ldots,Z_n\bigg] &\le n \bigg(
		2\sqrt{2}e \sum_{k=1}^\infty \log\Big(\mathcal{N}\big(\Q, L^2_\xi, D/2^k\big)^2 + 1\Big) \dfrac{3D}{2^{k-1}} + D \bigg)\\
		&\le n \bigg(24\sqrt{2}e \sum_{k=1}^\infty \int_{\frac{D}{2^{k+1}}}^{\frac{D}{2^k}} \log\Big(\mathcal{N}\big(\Q, L^2_\xi, t\big)^2 + 1\Big) \, dt + D \bigg)\\
		&= n \bigg(24\sqrt{2}e \int_{0}^{\frac{D}{2}} \log\Big(\mathcal{N}\big(\Q, L^2_\xi, t\big)^2 + 1\Big) \, dt + D \bigg)\\
        &\le n \bigg(72\sqrt{2}e \int_{0}^{\frac{D}{2}} \log\big(\mathcal{N}\big(\Q, L^2_\xi, t\big)\big) \, dt + D \bigg),
	\end{align*}
    where in the last inequality we have used the inequality $\mathcal{N}\big(\Q,L^2_\xi,t)^2 + 1 \le \mathcal{N}\big(\Q,L^2_\xi,t)^3$ since $\mathcal{N}\big(\Q,L^2_\xi,t) \ge 2$ for $t \le \frac{D}{2}$.
This completes the proof.
\end{proof}
    
    Applying Lemma \ref{lemma:chaos} with $\Q = \W = \{\hat{h}_f: f\in\calH\}$, we know
	\begin{align}\label{eq:bound_Zepsilon}
		\mathbb{E} \big[Z_\epsilon\big|Z_1,\ldots,Z_n\big] &\le n\bigg(72\sqrt{2}e \int_{0}^{F} \log\Big(\mathcal{N}\big(\W,L^2_\xi, t\big)\Big) \, dt + F\bigg),
	\end{align} 
    where $F:= \sup_{f\in\calH} \big\|\hat{h}_f\big\|_{L^\infty(\Z\times\Z)} \ge D$.
	
    Now we estimate $\mathbb{E}[U_\epsilon]$. Denote by $H_f \in \R^{n\times n}$ a square matrix with zero diagonal such that its entries $(H_f)_{i,j} = \hat{h}_f(Z_i,Z_j)$ for $i\neq j$, and let $\epsilon = (\epsilon_1,\ldots,\epsilon_n) \in \R^n$ be i.i.d. Rademacher variables. Then, we know
    \begin{align}\label{ineq:Uepsilon}
        \bE[U_\epsilon^2] &= \bE\Big[\Big(\sup_{f \in\calH}\sup_{\|\alpha\|_2 \le 1} \alpha^\top H_{f} \epsilon\Big)^2\Big]\nonumber\\
        &= \bE\Big[\Big(\sup_{f \in\calH}\|H_{f} \epsilon\|_2\Big)^2\Big]\nonumber\\
        &\le \bE\Big[\sup_{f \in\calH}\|H_{f} \epsilon\|_2^2\Big]\nonumber\\
        &= \bE\Big[\sup_{f \in\calH} \epsilon^\top H_{f}^2 \epsilon\Big].
    \end{align}
    From above, we know the estimates of $\bE[U_\epsilon]$ reduce to the estimates of the Rademacher chaos. Indeed, we first suppose that $\Q = \{J_f: f\in\calH\}$ is any class such that \begin{align}\label{def:J_f}
        J_f(Z_i,Z_j) = \big(H^2_f\big)_{i,j} = \sum_{k\neq i, k \neq j}^n \hat{h}_f(Z_i,Z_k) \hat{h}_f(Z_j,Z_j)
    \end{align} for all $f \in\calH$. Then, there holds
    \begin{align}\label{equality:Uepsilon}
        \bE\Big[\sup_{f \in\calH} \epsilon^\top H_{f}^2 \epsilon\Big] &= \bE\Big[\sup_{f \in\calH} \sum_{i,j = 1}^n\epsilon_i\epsilon_j J_f(Z_i,Z_j)\Big]\nonumber\\
        &= \bE\Big[\sup_{f \in\calH} \sum_{i\neq j}^n\epsilon_i\epsilon_j J_f(Z_i,Z_j)\Big] + \bE\Big[\sup_{f \in\calH} \sum_{i\neq j}^n \hat{h}_f^2(Z_i,Z_j)\Big].
    \end{align}
    The first term of the above equality is a Rademacher chaos, which will be estimated using Lemma \ref{lemma:chaos}. The second term equals $\bE\big[\sup_{f\in\calH}\|\hat{h}_f\|_{L^2_\xi}\big]$, which will be controlled by the uniform boundedness of $\hat{h}_f$ directly. Combining \eqref{ineq:Uepsilon} and \eqref{equality:Uepsilon} together, and bounding the covering number of $\Q$ by that of $\W$, we can control $\bE[U_\epsilon]$ as follows

	\begin{lemma}\label{lemma:Uepsilon} Let $D = \sup_{f\in\calH}\big\|\hat{h}_f\big\|_{L^2_\xi}$, there holds
    $$\mathbb{E}\big[U_\epsilon\big|Z_1,\ldots,Z_n\big] \le CnD\sqrt{\int_{0}^{1/2} \log\big(\mathcal{N}\big(\W, L^2_\xi, t\big)\big) \, dt + 1}.$$
    \end{lemma}
 	\begin{proof} According to \eqref{ineq:Uepsilon} and \eqref{equality:Uepsilon}, there holds
    \begin{align*}
        \bE[U_\epsilon^2|Z_1,\ldots,Z_n] &= \bE\Big[\sup_{f \in\calH} \sum_{i\neq j}^n\epsilon_i\epsilon_j J_f(Z_i,Z_j)\Big|Z_1,\ldots,Z_n\Big] + \bE\Big[\sup_{f \in\calH} \sum_{i\neq j}^n \hat{h}_f^2(Z_i,Z_j)\Big| Z_1,\ldots,Z_n\Big]\\
        &\le n \bigg(72\sqrt{2}e \int_{0}^{D/2} \log\big(\mathcal{N}\big(\Q, L^2_\xi, t\big)\big) \, dt + D\bigg) + n^2 D\\
        &\le C n^2D\Big(\int_{0}^{D/2} \log\big(\mathcal{N}\big(\Q, L^2_\xi, t\big)\big) \, dt + 1\Big),
    \end{align*}
    where in the first inequality we have used Lemma \ref{lemma:chaos} and the fact $D = \sup_{f\in\calH} \|\hat{h}_f\|_{L^2_\xi}$.
    
    Now, we are in a position to estimate the covering number $\mathcal{N}\big(\Q, L^2_\xi, t\big)$. Denote by $\|H\|_F = \sqrt{\sum_{i,j=1}^2(H_{i,j})^2}$ the Frobenius norm of a matrix $H$. For any $f\in\calH$, the $L^2_\xi$ norm of $J_f$ (defined in \eqref{def:J_f}) can be bounded as follows
    \begin{align*}
        \big\|J_f\big\|_{L^2_\xi} &= \bigg(\biavg J_f^2(Z_i,Z_j)\bigg)^{\frac{1}{2}} \le \bigg(\frac{1}{n(n-1)}\sum_{i,j=1}^n J_f^2(Z_i,Z_j)\bigg)^{\frac{1}{2}}\\
        &= \frac{1}{\sqrt{n(n-1)}} \big\|H_f^2\big\|_F \le \frac{1}{\sqrt{n(n-1)}} \big\|H_f\big\|_F^2 =\big\|\hat{h}_f\big\|^2_{L^2_\xi}\\
        &\le D\big\|\hat{h}_f\big\|_{L^2_\xi}.
    \end{align*}
    Therefore, we know $\mathcal{N}\big(\Q,L^2_\xi,t\big) \le \mathcal{N}\big(\W ,L^2_\xi,t/D\big)$ for any $t\in(0,D/2)$. Then, there holds
    \begin{align*}
        \bE[U_\epsilon^2|Z_1,\ldots,Z_n] &\le C n^2D\Big(\int_{0}^{D/2} \log\big(\mathcal{N}\big(\W, L^2_\xi, t/D\big)\big) \, dt + 1\Big)\\
        &= C n^2D^2\Big(\int_{0}^{1/2} \log\big(\mathcal{N}\big(\W, L^2_\xi, t\big)\big) \, dt + 1\Big).
    \end{align*}
    By Jensen's inequality for conditional expectation, we know $\bE[U_\epsilon|Z_1,\ldots,Z_n] \le \big(\bE[U_\epsilon^2|Z_1,\ldots,Z_n]\big)^{1/2}$. Then, we finally conclude that
    \begin{align*}
        \bE\big[U_\epsilon|Z_1,\ldots,Z_n\big] \le CnD\sqrt{\int_{0}^{1/2} \log\big(\mathcal{N}\big(\W, L^2_\xi, t\big)\big) \, dt + 1}.
    \end{align*}
    The proof of this lemma is complete.
    \end{proof}

    For the last term $\bE[M]$, we can control it just by the uniform boundedness of the class $\W$. Indeed, recall that $F = \sup_{f\in\calH} \big\|\hat{h}_f\big\|_{L^\infty(\Z\times\Z)}$, then we immediately get the upper bounds of $\bE[M]$ as follows
    \begin{align}\label{ineq:M}
        \bE[M] \le \sup_{f\in\calH,k=1,\ldots,n} \sum_{i=1}^n\big| \hat{h}_f(Z_i,Z_k)\big| \le nF.
    \end{align}
    From \eqref{eq:bound_Zepsilon} and Lemma \ref{lemma:Uepsilon}, we know the only thing left now is the estimates of the covering number $\mathcal{N}\big(\W,L^2_\xi,t\big).$ Combining all the above results, and carefully controlling $\mathcal{N}\big(\W,L^2_\xi,t\big)$ in terms of the capacity of the hypothesis space $\calH$, we can establish an upper bound for $S_2(\calH)$ as the following proposition.

	\begin{lemma}\label{prop:S_2}
        Suppose Assumptions 1, 2, and $\ell(f(x,x'),y,y') = \ell(f(x',x),y',y)$ for almost $z,z'\in\Z$ hold, and the hypothesis space $\calH$ is uniformly bounded by $\eta > 0$.
        \begin{itemize}
            \item  If the capacity of $\calH$ satisfies Assumption 3, then for any $\delta \in (0, 1/2)$, with probability at least $1 - \delta/2$, there holds
            \begin{align*}
                S_2(\calH) \le C_{\eta, K}\max\{V_1,V_2,\log(s_1),\log(s_2)\}\frac{\log^2(\delta/2)}{n}.
            \end{align*}
         \item    If the capacity of $\calH$ satisfies Assumption 4, then for any $\delta\in(0,1/2)$, with probability at least $1 - \delta/2$, there holds
		$$S_2(\calH) \le C_{\eta, K} \max\{s_1',s_2'\}\frac{\log^2(\delta/2)}{n}\frac{1}{1 - \max\{V_1',V_2'\}}.$$
        \end{itemize}
	\end{lemma}
    The following lemma studies the covering numbers of the sum of the function classes, which will be used in the proof of Lemma \ref{prop:S_2}.
    \begin{lemma}\label{lemma:sum_covering}
        Let $\big(\Omega, \Gamma, p\big)$ be a measure space. Suppose $\W_1,\ldots,\W_m \subset L^2\big(\Omega, \Gamma, p\big)$ and define $\bigoplus_{i=1}^m\W_i := \big\{\sum_{i=1}^m w_i: w_i\in\W_i\big\}$ as the direct sum of the classes $\W_i$ for $i=1\ldots,m$. Then, for any $t>0$, the covering numbers of the above function classes satisfying
        \begin{align*}
            \cN\Big(\bigoplus_{i=1}^m \W_i,L^2_p,t\Big) \le \prod_{i=1}^m\cN\big(\W_i, L^2_p, t/m\big).
        \end{align*}
    \end{lemma}
    \begin{proof}
        Let $N_i \subset \W_i$ be a $\frac{t}{m}$-net of $\W_i$ for $i=1,\cdots,m$. Then, we claim that $N= \bigoplus_{i=1}^m N_i := \{\sum_{i=1}^m w^*_i: w^*_i\in\N_i\}$ is a $t$-net of $\bigoplus_{i=1}^m\W_i$. Indeed, for any $\sum_{i=1}^m w_i \in \bigoplus_{i=1}^m \W_i$, we know there exists a $w^*_i \in \W_i$ such that $\|w_i - w^*_i\|_{L^2_p} \le t/m$ for $i=1,\ldots,m$. Then, we have $\|\sum_{i=1}^m w_i - \sum_{i=1}^m w^*_i\|_{L^2_p} \le \sum_{i=1}^m \|w_i - w^*_i\|_{L^2_p} \le t$, which means $N$ is $t$-net of $\bigoplus_{i=1}^m\W_i$. Further, notice that $|N| \le \prod_{i=1}^m |N_i|$ and $N_i$ is an arbitrary $\frac{t}{m}$-net of $\W_i$. Then, the proof of this lemma is complete.
    \end{proof}
	The proof of Lemma \ref{prop:S_2} is given as follows.
	\begin{proof}[Proof of Lemma \ref{prop:S_2}] Recall that by Lemma \ref{lemma:ranking}, we know with probability at least $1 - \delta/2$, there holds
	$$S_2(\calH) \le \frac{C\log^2(2/\delta)}{n^2} \Big(\mathbb{E}[Z_\epsilon] + \mathbb{E}[U_\epsilon] + \mathbb{E}[M] + Fn \Big).$$
    Notice that the constant $F = \sup_{f\in\calH} \big\|\hat{h}_f\big\|_{L^\infty(\Z\times\Z)} \le 4\sup_{f\in\calH} \big\|q_f\big\|_{L^\infty(\Z\times\Z)} \le 4K\sup_{f\in\calH} \big\|f - f_\rho\big\|_{L^\infty(\X\times\X)}\le 8\eta K$. According to \eqref{eq:bound_Zepsilon}, Lemma \ref{lemma:Uepsilon}, and \eqref{ineq:M}, with probability at least $1 - \delta/2$, there holds
    \begin{align*}
        S_2(\calH) \le C_{\eta,K}\frac{1}{n} \bE\bigg[\bigg(\int_{0}^{8\eta K} \log\Big(\mathcal{N}\big(\W,L^2_\xi, t\big)\Big) \, dt + \sqrt{\int_{0}^{1/2} \log\big(\mathcal{N}\big(\W, L^2_\xi, t\big)\big) \, dt} + 1\bigg)\bigg].
    \end{align*}
    
    Now, we are in a position to estimate the covering number $\mathcal{N}\big(\W,L^2_\xi, t\big)$. We first define three classes of functions from $\Z\times\Z$ to $\R$ as follows
    \begin{align*}
        \W_1 &= \big\{w^1_f(z,z') = \bE[q_f(Z,Z')]: f\in\calH\big\}\\
        \W_2 &= \big\{w^2_f(z,z') = -h_f(z) - h_f(z'): f\in\calH\big\}\\
        \W_3 &= \big\{w^3_f(z,z') = q_f(z,z'): f \in \calH\big\}.
    \end{align*}
    It can be seen that for any $f \in\calH$, $\hat{h}_f = w^1_f + w^2_f + w^3_f$. Then, it follows that $\W \subset \W_1 + \W_2 + \W_3$. According to Lemma \ref{lemma:sum_covering}, for any $t > 0$ and Exercise 4.2.10 in \cite{HDP}, there holds
    \begin{align*}
        \mathcal{N}\big(\W,L^2_\xi,t\big) \le \mathcal{N}\big(\W_1,L^2_\xi,t/6\big)\mathcal{N}\big(\W_2,L^2_\xi,t/6\big)\mathcal{N}\big(\W_3,L^2_\xi,t/6\big).
    \end{align*}
    We will estimate $\cN\big(\W_i,L^2_\xi,t\big)$ for $i = 1,2,3$ separately. For any $f_1, f_2\in\calH$, there hold
    \begin{align*}
        \big\|w^1_{f_1} - w^1_{f_2}\big\|_{L^2_\xi} &= \big|\bE\big[\ell(f_1(X,X'),Y,Y') - \ell(f_2(X,X'),Y,Y')\big]\big| \le K\big\|f_1 - f_2\big\|_{L^\infty(\X\times\X)},\\
        \big\|w^2_{f_1} - w^2_{f_2}\big\|_{L^2_\xi} &\le \big\|h_{f_1}(z) - h_{f_2}(z)\big\|_{L^2_\xi} + \big\|h_{f_1}(z') - h_{f_2}(z')\big\|_{L^2_\xi} = 2\Big(\frac{1}{n}\sum_{i=1}^n \big|h_{f_1}(Z_i) - h_{f_2}(Z_i)\big|^2\Big)^{1/2}\\
        &\le 2K\Big(\avg \bE\Big[\big|f_1(X,X_i) - f_2(X,X_i)\big|^2\Big| X_i\Big]\Big)^{1/2} = 2K\big\|f_1 - f_2\big\|_{L^2_{\rho_{\bx}\times\mu_n}},\\
        \big\|w^3_{f_1} - w^3_{f_2}\big\|_{L^2_\xi} &= \Big(\biavg \big|q_{f_1}(Z_i,Z_j) - q_{f_2}(Z_i,Z_j)\big|^2\Big)^{1/2}\\
        &\le K\Big(\biavg \big|f_1(X_i,X_j) - f_2(X_i,X_j)\big|^2\Big)^{1/2}\\
        & = K\big\|f_1 - f_2\big\|_{L^2_{\nu_n}}.
    \end{align*}
    Therefore, for any $t>0$, we have
    \begin{align}\label{ineq:covering_W}
        \mathcal{N}\big(\W,L^2_\xi,t\big) &\le \mathcal{N}\big(\calH,L^\infty(\X\times\X),t/6K\big)\mathcal{N}\big(\calH,L^2_{\rho_{\bx}\times\mu_n},t/12K\big)\mathcal{N}\big(\calH,L^2_{\nu_n},t/6K\big)\nonumber\\
        &\le \cN\big([-8\eta K, 8\eta K], |\cdot|, t/6K\big) \mathcal{N}\big(\calH,L^2_{\rho_{\bx}\times\mu_n},t/12K\big)\mathcal{N}\big(\calH,L^2_{\nu_n},t/6K\big)\nonumber\\
        &\le C_{\eta,K} \frac{1}{t} \mathcal{N}\big(\calH,L^2_{\rho_{\bx}\times\mu_n},t/12K\big)\mathcal{N}\big(\calH,L^2_{\nu_n},t/6K\big).
    \end{align}
    {\bf First case}: Suppose Assumption 3 holds.

    From $\eqref{ineq:covering_W}$, we know
    \begin{align*}
        \log\Big(\cN\big(\W,L^2_\xi,t\big)\Big) &\le \log(C_{\eta,K}) + \log\Big(\frac{1}{t}\Big) + \log(s_1) + V_1\log\Big(\frac{12K}{t}\Big) + \log(s_2) + V_2\log\Big(\frac{6K}{t}\Big)\\
        &\le C_{\eta,K} \max\{V_1,V_2,\log(s_1),\log(s_2)\}\log\Big(\frac{12K}{t}\Big).
    \end{align*}
    Then, with probability at least $1 - \delta/2$, $S_2(\calH)$ can be bounded as
    \begin{align*}
        S_2(\calH) &\le C_{\eta,K} \frac{\log^2(\delta/2)}{n}\bigg(\max\{V_1,V_2,\log(s_1),\log(s_2)\}\int_0^{8\eta K} \log\Big(\frac{12K}{t}\Big) dt\\
        &\ \ \ \ \ \ \ +\sqrt{\max\{V_1,V_2,\log(s_1),\log(s_2)\}\int_0^{1/2}\log\Big(\frac{12K}{t}\Big) dt} + 1\bigg)\\
        &\le C_{\eta,K} \frac{\log^2(\delta/2)\max\{V_1,V_2,\log(s_1),\log(s_2)\}}{n}.
    \end{align*}
    
    {\bf Second case}: Suppose Assumption 4 holds.

    From $\eqref{ineq:covering_W}$, we know
    \begin{align*}
        \log\Big(\cN\big(\W,L^2_\xi,t\big)\Big) &\le \log(C_{\eta,K}) + \log\Big(\frac{1}{t}\Big) + s_1'\Big(\frac{12K}{t}\Big)^{V_1'} + s_2'\Big(\frac{6K}{t}\Big)^{V_2'}\\
        &\le C_{\eta,K}\max\{s_1',s_2'\} \Big(\frac{12K}{t}\Big)^{\max\{V_1',V_2'\}}.
    \end{align*}
    Then, with probability at least $1 - \delta/2$, $S_2(\calH)$ can be bounded as
    \begin{align*}
       & S_2(\calH)\nonumber\\
       &\!\le\! C_{\eta,K} \frac{\log^2(\delta/2)}{n}\bigg(\!\int_0^{8\eta K}\!\!\!\max\{s_1',s_2'\}\Big(\frac{12K}{t}\Big)^{\max\{V_1',V_2'\}} dt + \sqrt{\!\int_0^{1/2}\!\!\!\max\{s_1',s_2'\}\Big(\frac{12K}{t}\Big)^{\max\{V_1',V_2'\}}\!dt} \!+\! 1\bigg)\\
        &\!=\! C_{\eta,K} \max\{s_1',s_2'\}\frac{\log^2(\frac{\delta}{2})}{n}\bigg(\frac{(12K)^{\max\{V_1',V_2'\}}(8\eta K)^{1 \!-\! \max\{V_1',V_2'\}}}{1 - \max\{V_1',V_2'\}} \!+\! \sqrt{\frac{(12K)^{\max\{V_1',V_2'\}}(\frac{1}{2})^{1 \!-\! \max\{V_1',V_2'\}}}{1 - \max\{V_1',V_2'\}}} \!+\! 1\bigg)\\
        &\!\le\!  C_{\eta,K} \max\{s_1',s_2'\}\frac{\log^2(\frac{\delta}{2})}{n}\frac{1}{1 - \max\{V_1',V_2'\}}.
    \end{align*}    
 This completes the proof of Lemma \ref{prop:S_2}.
 \end{proof}

 \subsection{Proof of Theorem 1}
	Now, we are in the position to present the proof of our main result.
	\begin{proof}[Proof of Theorem 1.] Notice that $\bE\big[g_{\hat{f}_z}\big] = \bE\big[\ell(\hat{f}_z(X,X'),Y,Y') - \ell(\hat{f}_\rho(X,X'),Y,Y')\big] = \E(\hat{f}_z) - \E(f_\rho)$ is the excess generalization error. Combining Lemma \ref{prop:S_1} and Lemma \ref{prop:S_2} together, there holds
     
      \noindent(a) If the capacity of $\calH$ satisfies Assumption 3, then with probability at least $1 -\delta$, there holds
        \begin{align*}
            &\E(\hat{f}_z) - \E(f_\rho)\nonumber\\ &\le C_{\eta, K,M,\beta} \Big(\frac{\max\{\log(s_1),V_1\}\log(n)}{n}\Big)^{\frac{1}{2-\beta}}\log(4/\delta) + C_{\eta, K} \frac{\log^2(\delta/2)\max\{V_1,V_2,\log(s_1),\log(s_2)\}}{n}\\
            &\ \ \ \ + \frac{1}{2}\big(\E(\hat{f}_z) - \E(f_\rho)\big) + \Big(\frac{\beta}{2} + 1\Big) \D(\calH).
        \end{align*}
       (b) If the capacity of $\calH$ satisfies Assumption 4, then with probability at least $1 -\delta$, there holds
        \begin{align*}
            &\E(\hat{f}_z) - \E(f_\rho)\nonumber\\
            &\le C_{\eta, K,M,\beta} \max\bigg\{ \Big(\sqrt{s_1'}\frac{1}{n}\Big)^{\frac{2}{(2+V_1')(2-\beta)}}, \Big(\frac{\log(n)}{n}\Big)^{\frac{1}{2-\beta}}\log(4/\delta)\bigg\} + C_{\eta, K}\max\{s_1',s_2'\} \frac{\log^2(\delta/2)}{n\big(1 - \max\{V_1',V_2'\}\big)}\\
            &\ \ \ \ + \frac{1}{2}\big(\E(\hat{f}_z) - \E(f_\rho)\big) + \Big(\frac{\beta}{2} + 1\Big) \D(\calH).
        \end{align*}
   
    By rearranging the terms of the above inequalities, we can obtain the desired results.\end{proof}

\section{Proofs for generalization analysis with deep ReLU networks }\label{sec:proof-relu}
The proof of Proposition 1 is given as follows.  
\begin{proof}[Proof of Proposition 1]
        By tower property of the conditional expectation, the generalization error with a predictor $f$ can be written as
        \begin{align*}
            \E(f) = \bE\big[\bE\big[\ell\big(f(X,X'),Y,Y'\big)\big|X,X'\big]\big],
        \end{align*}
        which implies that the true predictor $f_\rho(x,x')$ is obtained by minimizing the inner conditional expectation for almost $x,x'\in\X$.  Then, there holds
		\begin{align*}
			f_\rho(x,x') &= \arg\min_{t\in\R} \bE\big[\ell(t,Y,Y')\big|X=x,X'=x'\big]\\
            &= \arg\min_{t\in\R} \int_{\Y\times\Y} \ell(t, y, y') \, d\rho(y|x) \, d\rho(y'|x').
        \end{align*}
        According to Assumption 5, we further know that 
        \begin{align*}
			f_\rho(x,x') &= \arg\min_{t\in\R} \int_{\Y\times\Y} \ell(-t, y', y) \, d\rho(y'|x') \, d\rho(y|x)\\
			&= -\arg\min_{t\in\R} \int_{\Y\times\Y} \ell(t, y', y) \, d\rho(y'|x') \, d\rho(y|x)\\
            &= -\arg\min_{t\in\R} \bE\big[\ell(t,Y',Y)\big|X'=x',X=x\big]\\
			&= -f_\rho(x',x).
		\end{align*}
  The second part of the proposition can be directly derived from $f_\rho(x,x')= -f_\rho(x',x).$ 
		The proof is completed.
	\end{proof}
    The proof of Theorem 2 can be directly derived from Theorem 1 by verifying the corresponding assumptions. 
    \begin{proof}[Proof of Theorem 2]
        From \cite{VC} we know the space $\calH$ of deep ReLU networks is a VC-class. Then, Lemma 1 implies that $\calH$ satisfies Assumption 3 with $V_1 = V_2 = 2(V - 1)$ and $s_1 = s_2 = C(V/2 + 1)(16e)^{V/2 + 1}\eta^V$. Further, Assumption 5 and the anti-symmetric structure of $\calH$ and $f_\rho$ imply $\ell(f(x,x'),y,y') = \ell(f(x',x),y',y)$. Then, we can get the desired results by applying Theorem 1 directly.
    \end{proof}
To prove Theorem 3, we first introduce the following lemma which shows the expressive ability of deep ReLU networks for approximating functions in the Sobolev spaces.
    \begin{lemma}[\cite{yarotsky}]\label{lemma:approx_Sobolev}
        For any $d, r \in \N$, $\epsilon \in (0,1/2)$ and any function $f \in W^{r,\infty}([0,1]^d)$ with Sobolev norm not larger than $1$, there exists a deep ReLU network $h$ with depth at most $C_{d,r} \log(1/\epsilon)$ and the number of nonzero weights and computational units at most $C_{d,r} \epsilon^{-\frac{d}{r}} \log(1/\epsilon)$ such that
        \begin{align*}
            \|h - f\|_{L^\infty([0,1]^d)} \le \epsilon.
        \end{align*}
    \end{lemma}
    Now, we  give the proof of Theorem 3 on the estimate of the approximation error. 
   \begin{proof}[Proof of Theorem 3]
        According to Lemma \ref{lemma:approx_Sobolev}, we know there exists a deep ReLU network $h$ with depth at most $C_{d,r} \log(1/\epsilon)$ and the number of nonzero weights and computational units at most $C_{d,r} \epsilon^{-\frac{d}{r}} \log(1/\epsilon)$ such that $\big\|h - \tilde{f}_\rho\big\|_{L^\infty([0,1]^d)} \le \frac{\epsilon}{2}$. We construct $f$ as 
        \begin{center}
            $f(x,x') = \pi_\eta\big(h(x)\big) - \pi_\eta\big(h(x')\big)$ for $x,x'\in\X$.
        \end{center}
        Here, we suppose $\eta = 2$ since Assumption 6 implies Assumption 1 with $\eta = 2$. We can regard $h(x)$ and $h(x')$ as functions defined on $[0,1]^d\times[0,1]^d$, which are denoted by $(x,x')\mapsto h(x)$ and $(x,x')\mapsto h(x')$, respectively. Therefore, we know $f$ is a deep ReLU network with depth at most $C_{d,r} \log(1/\epsilon)$ and the number of nonzero weights and computational units at most $C_{d,r} \epsilon^{-\frac{d}{r}} \log(1/\epsilon)$. The approximation accuracy can be bounded as follows
        \begin{align*}
            \big\|f - f_\rho\big\|_{L^\infty([0,1]^{2d})} &= \big\|\pi_\eta\big(h(x)\big) - \pi_\eta\big(h(x')\big) - \tilde{f}_\rho(x) + \tilde{f}_\rho(x')\big\|_{L^\infty([0,1]^{2d})}\\
            &\le 2\big\|\pi_\eta\big(h(\cdot)\big) - \tilde{f}_\rho(\cdot)\big\|_{L^\infty([0,1]^d)}\\
            &\le 2\big\|h - \tilde{f}_\rho\big\|_{L^\infty([0,1]^d)} \le \epsilon,
        \end{align*}
        where in the second inequality we have used the fact $\pi_\eta\big(h(\cdot)\big)$ is uniformly bounded by $1$. This completes the proof of the first inequality.

        From \cite{online_pairwise} we know for any $T \in L^2_{\rho_{\bx}^2}(\X\times\X)$, the excess risk $\E(T) - \E(f_\rho) = \big\|T - f_\rho\big\|^2_{L^2_{\rho_{\bx}^2}}$. Then, the excess risk of $f$ can be controlled as follows
        \begin{align*}
            \E(f) - \E(f_\rho) &= \big\|f - f_\rho\big\|^2_{L^\infty([0,1]^{2d})}\le \epsilon^2.
        \end{align*}
        According to the complexity of $f$, by setting parameters $W = U = \lceil \exp{L} \rceil$, we have
        $\E(f) - \E(f_\rho) \le C_{d,r}\big(\frac{L}{\exp(L)}\big)^{\frac{2r}{d}}$. Then, the approximation error
        \begin{align*}
            \D(\calH) = \inf_{f \in \calH} \E(f) - \E(f_\rho) \le \E(f) - \E(f_\rho) \le C_{d,r}\Big(\frac{L}{\exp(L)}\Big)^{\frac{2r}{d}},
        \end{align*}
        which completes the proof.
    \end{proof}

\begin{lemma}\label{lemma:bounded_ditri}
        Suppose Assumption 7 holds. Then, Assumption 1 holds with $\eta = 2B$, Assumption 2 holds with $K=8B$, and the shifted hypothesis space $\{\ell\big(f(x,x'), y, y'\big) - \ell\big(f_\rho(x,x'), y, y'\big) : f \in \calH\}$ has a variance-expectation bound with parameter pair $(1, 64B^2)$.
	\end{lemma}
\begin{proof}
            Notice $\tilde{f}_\rho(x) = \bE[Y|X = x]$ for almost $x\in\X$. Then, $\|\tilde{f}_\rho\|_{L^\infty(\X)} \le \|Y\|_{L^\infty(\Y)} \le B$. Since $f_\rho(x,x') = \tilde{f}_\rho(x) - \tilde{f}_\rho(x')$, we know $\|f_\rho\|_{L^\infty(\X\times\X)} \le 2\|\hat{f}_\rho\|_{L^\infty(\X)} \le 2B$. It follows that Assumption 1 holds with $\eta = 2B$.

            For any $t_1,t_2\in[-2B,2B]$, and almost $y,y'\in\Y$,
            \begin{align*}
                \big|\ell\big(t_1,y,y'\big) - \ell\big(t_2,y,y'\big)\big| &\le \Big|\big(t_1 - y + y'\big)^2 - \big(t_2 - y + y'\big)^2\Big|\\
                &\le \big|t_1 - t_2\big|\big|t_1+t_2 - 2y + 2y'\big|\\
                &\le 8B\big|t_1 - t_2\big|.
            \end{align*}
            Then, Assumption 2 holds with $K = 8B$.

            For any $f\in\calH$, we define $q_f = \big(f(x,x') - y + y'\big)^2 - \big(f_\rho(x,x') - y + y'\big)^2$. Then
            \begin{align*}
                \bE\big[q_f^2\big] &= \bE\Big[\Big(\big(f(X,X')- Y + Y'\big)^2 - \big(f_\rho(X,X')- Y + Y'\big)^2\Big)^2\Big]\\
                &= \bE\big[\big(f(X,X') - f_\rho(X,X')\big)^2\big(f(X,X') + f_\rho(X,X') - 2Y + 2Y'\big)^2\big]\\
                &\le 64B^2 \bE \big[\big(f(X,X') - f_\rho(X,X')\big)^2\big]\\
                &= 64B^2\big(\E(f) - \E(f_\rho)\big) = 64B^2\bE[q_f],
            \end{align*}
            where the third equality is due to the specialty of the least squares loss \cite{online_pairwise}. Hence, we know the shifted hypothesis space has a variance-expectation bound with parameter pair $(1,64B^2)$.
            This completes the proof.
	\end{proof}

The excess generalization error bound is obtained by combining Theorem 2 and Theorem 3 together. 
\begin{proof}[Proof of Theorem 4]
        According to  Theorem 7 in \cite{VC}, we know the $\calH$ is a VC-class and its pseudo-dimension $V = Pdim(\calH) \le CLW\log(U)$. Lemma \ref{lemma:bounded_ditri} implies all the conditions in Theorem 2 are satisfied. Then, combining Theorem 2 and Theorem 3 together, and setting $W = U = \lceil \exp(L)\rceil$, we know with probability at least $1 - \delta$, there holds
        \begin{align*}
            \E(\hat{f}_z) - \E(f_\rho) \le C_B \frac{L^2\exp(L)\log(n)}{n}\log^2(4/\delta) + C_{d,r} \Big(\frac{L}{\exp(L)}\Big)^{\frac{2r}{d}}.
        \end{align*}
        This proves the first part of the theorem.

        By setting $L = \lceil \frac{d}{2r+d} \log(n)\rceil$, we get the desired rate immediately. The proof of the theorem is complete.
    \end{proof}